\documentclass{article} 
\usepackage{nips11submit_e,times}

\usepackage{epsf}
\usepackage{graphicx}
\usepackage{amsmath}
\usepackage{amssymb}
\usepackage{algorithmic}
\usepackage{algorithm}
\usepackage{array}
\usepackage{amsmath}
\usepackage{amssymb}
\usepackage{amsfonts}
\usepackage{amsthm}
\usepackage{booktabs}
\usepackage{caption}

\newlength{\minipagewidth}
\newcommand{\bookbox}[1]{\small
\par\medskip\noindent
\framebox[\columnwidth]{
\begin{minipage}{\minipagewidth} {#1} \end{minipage} } \par\medskip }

\renewcommand{\Re}{\mathbb R}

\def\argmin{\mathop{\rm arg\,min}}

\newcommand{\MDP}{\mathcal M}
\newcommand{\bMDP}{\overline{\MDP}}
\newcommand{\discount}{\gamma}
\newcommand{\state}{\mathcal{X}}
\newcommand{\action}{\mathcal{A}}
\newcommand{\X}{\mathcal{X}}
\newcommand{\A}{\mathcal{A}}
\newcommand{\XA}{{\X\times\A}}
\newcommand{\reward}{\mathcal{R}}
\newcommand{\dynamics}{\mathcal{P}}
\newcommand{\bReward}{\overline{\mathcal{R}}}
\newcommand{\bDynamics}{\overline{\mathcal{P}}}
\newcommand{\T}{{\mathcal{T}}}
\newcommand{\bT}{{\overline{\mathcal{T}}}}

\newcommand{\nSamples}{L}

\newcommand{\nTasks}{M}
\newcommand{\sumTasks}{\sum_{m=1}^{\nTasks}}

\newcommand{\nStates}{N}
\newcommand{\sumStates}{\sum_{n=1}^{\nStates}}
\newcommand{\avgStates}{\frac{1}{\nStates}\sumStates}

\newcommand{\nStatesTask}{N_{m}}

\newcommand{\nTot}{\nSamples}
\newcommand{\sumTot}{\sum_{l=1}^{\nTot}}
\newcommand{\avgTot}{\frac{1}{\nTot}\sumTot}

\newcommand{\nDist}{S}
\newcommand{\sumDist}{\sum_{s=1}^{\nDist}}
\newcommand{\avgDist}{\frac{1}{\nDist}\sumDist}

\newcommand{\avgRatioB}{\sumTasks \lambda_{m}}
\newcommand{\avgRatioC}{\sum_{m=2}^\nTasks \lambda_{m}}

\newcommand{\nActions}{|\action|}
\newcommand{\sumActions}{\sum_{a\in\action}}
\newcommand{\avgActions}{\frac{1}{\nActions}\sumActions}

\newcommand{\taskDist}{\mathcal{E}}
\newcommand{\avgTaskDist}{\mathcal{E}}
\newcommand{\hTaskDist}{\widehat{\mathcal{E}}}

\newcommand{\F}{\mathcal{F}}

\newcommand{\B}{\mathcal{B}}

\newcommand{\mSpace}{\mathcal{S}}

\newcommand{\Rmax}{R_{\max}}

\newcommand{\Vmax}{V_{\max}}

\newcommand{\hQ}{\widehat{Q}}
\newcommand{\hQall}{\hQ}
\newcommand{\hQbat}{\hQ_{\text{BAT}}}
\newcommand{\tQ}{\widetilde{Q}}

\newcommand{\trunc}{T}

\newcommand{\hlambda}{\widehat{\lambda}}

\newcommand{\halpha}{{\hat{\alpha}}}
\newcommand{\hbeta}{{\hat{\beta}}}

\newcommand{\ind}[1]{\mathbb I\left\lbrace {#1} \right\rbrace}

\newcommand{\norm}[1]{\Arrowvert{#1}\Arrowvert}

\newcommand{\noise}{\xi}
\newcommand{\hnoise}{\hat{\xi}}

\newcommand{\expectA}[1]{\mathbb E \left[ {#1} \right]}
\newcommand{\expectB}[2]{\mathbb E_{#1} \left[ {#2} \right]}

\newcommand{\iid}{\stackrel{iid}{\sim}}

\newcommand{\hPi}{\hat{\Pi}}

\newcommand{\hmu}{\hat{\mu}}

\newcommand{\paragTitle}[1]{\textbf{#1.}}

\newtheorem{lemma}{Lemma}
\newtheorem{assumption}{Assumption}

\newtheorem{definition}{Definition}
\newtheorem{theorem}{Theorem}

\newcommand{\otoprule}{\midrule[\heavyrulewidth]}

\title{Transfer from Multiple MDPs}

\author{
Alessandro Lazaric \\
INRIA Lille - Nord Europe, Team SequeL, France \\
\texttt{alessandro.lazaric@inria.fr} \\
\And
Marcello Restelli \\
Department of Electronics and Informatics, Politecnico di Milano, Italy \\
\texttt{restelli@elet.polimi.it}
}

\nipsfinalcopy 

\begin{document}

\maketitle


\begin{abstract}
Transfer reinforcement learning (RL) methods leverage on the experience collected on a set of source tasks to speed-up RL algorithms. A simple and effective approach is to transfer samples from source tasks and include them in the training set used to solve a target task. In this paper, we investigate the theoretical properties of this transfer method and we introduce novel algorithms adapting the transfer process on the basis of the similarity between source and target tasks. Finally, we report illustrative experimental results in a continuous chain problem.
\end{abstract}


\section{Introduction}\label{s:introduction}

The objective of transfer in reinforcement learning (RL)~\cite{sutton1998reinforcement} is to speed-up RL algorithms by reusing knowledge (e.g., samples, value function, features, parameters) obtained from a set of source tasks. The underlying assumption of transfer methods is that the source tasks (or a suitable combination of these) are somehow similar to the target task, so that the transferred knowledge can be useful in learning its solution. A wide range of scenarios and methods for transfer in RL have been studied in the last decade (see~\cite{taylor2009transfer,lazaric2008knowledge} for a thorough survey). In this paper, we focus on the simple transfer approach where trajectory samples are transferred from source MDPs to increase the size of the training set used to solve the target MDP. This approach is particularly suited in problems (e.g., robotics, applications involving human interaction) where it is not possible to interact with the environment long enough to collect samples to solve the task at hand. If samples are available from other sources (e.g., simulators in case of robotic applications), the solution of the target task can benefit from a larger training set that includes also some source samples. This approach has been already investigated in the case of transfer between tasks with different state-action spaces in~\cite{taylor2008transferring}, where the source samples are used to build a model of the target task whenever the number of target samples is not large enough. A more sophisticated sample-transfer method is proposed in~\cite{lazaric2008transfer}. The authors introduce an algorithm which estimates the similarity between source and target tasks and selectively transfers from the source tasks which are more likely to provide samples similar to those generated by the target MDP. Although the empirical results are encouraging, the proposed method is based on heuristic measures and no theoretical analysis of its performance is provided. On the other hand, in supervised learning a number of theoretical works investigated the effectiveness of transfer in reducing the sample complexity of the learning process. In domain adaptation, a solution learned on a source task is transferred to a target task and its performance depends on how \textit{similar} the two tasks are. In~\cite{ben-david2010a-theory} and~\cite{mansour2009domain} different distance measures are proposed and are shown to be connected to the performance of the transferred solution. The case of transfer of samples from multiple source tasks is studied in~\cite{crammer2008learning}. The most interesting finding is that the transfer performance benefits from using a larger training set at the cost of an additional error due to the average distance between source and target tasks. This implies the existence of a \textit{transfer tradeoff} between transferring as many samples as possible and limiting the transfer to sources which are similar to the target task. As a result, the transfer of samples is expected to outperform single-task learning whenever \textit{negative} transfer (i.e., transfer from source tasks far from the target task) is limited w.r.t. to the advantage of increasing the size of the training set. This also opens the question whether it is possible to design methods able to automatically detect the similarity between tasks and adapt the transfer process accordingly.

In this paper, we investigate the transfer of samples in RL from a more theoretical perspective w.r.t. previous works. The main contributions of this paper can be summarized as follows:

\begin{itemize}
\item \textit{Algorithmic contribution.} We introduce three sample-transfer algorithms based on fitted Q-iteration~\cite{ernst2005tree-based}. The first algorithm (\textit{AST} in Section~\ref{s:ast}) simply transfers all the source samples. We also design two adaptive methods (\textit{BAT} and \textit{BTT} in Section~\ref{s:average.bound} and ~\ref{s:tradeoff}) whose objective is to solve the transfer tradeoff by identifying the best combination of source tasks.
\item \textit{Theoretical contribution.} We formalize the setting of transfer of samples and we derive a finite-sample analysis of AST which highlights the importance of the \textit{average} MDP obtained by the combination of the source tasks. We also report the analysis for BAT which shows both the advantage of identifying the best combination of source tasks and the additional cost in terms of auxiliary samples needed to compute the similarity between tasks.
\item \textit{Empirical contribution.} We report results (in Section~\ref{s:experiments}) on a simple chain problem which confirm the main theoretical findings and support the idea that sample transfer can significantly speed-up the learning process and that adaptive methods are able to solve the transfer tradeoff and avoid negative transfer effects.
\end{itemize}

The rest of the paper is organized as follows. In Section~\ref{s:preliminaries} we introduce the notation and we define the transfer problem. Section~\ref{s:ast} reports the theoretical analysis of AST. BAT is described in Section~\ref{s:average.bound} along with its theoretical analysis. A more challenging setting is introduced in Section~\ref{s:tradeoff} together with BTT. Section~\ref{s:experiments} reports the experimental results and Section~\ref{s:conclusions} concludes the paper. Finally, in the appendix we report the proofs and some additional experimental analysis.


\section{Preliminaries}\label{s:preliminaries}


In this section we introduce the notation and the transfer problem considered in the rest of the paper.



\paragTitle{Notation for MDPs} We define a discounted Markov decision process (MDP) as a tuple $\MDP=\langle\X,\A,\reward,\dynamics,\discount\rangle$ where the state space $\X$ is a bounded closed subset of the Euclidean space, $\A$ is a finite ($|\A|<\infty$) action space, the deterministic\footnote{The extension to stochastic reward functions is straightforward.} 
reward function $\reward:\XA\rightarrow\Re$ is uniformly bounded by $\Rmax$, the transition kernel $\dynamics$ is such that for all $x\in\X$ and $a\in \A$, $\dynamics(\cdot|x,a)$ is a distribution over $\X$, and $\discount\in(0,1)$ is a discount factor. We denote by $\mSpace(\XA)$ the set of probability measures over $\XA$ and by $\B(\XA;\Vmax\!=\!\frac{R_{\max}}{1-\gamma})$ the space of bounded measurable functions with domain $\XA$ and bounded in $[-\Vmax,\Vmax]$. 
We define the optimal action-value function $Q^*$ as the unique fixed-point of the optimal Bellman operator $\T:\B(\XA;\Vmax)\rightarrow\B(\XA;\Vmax)$ defined by 

\begin{equation*}\label{opt-Bellman-op}
(\T Q)(x,a)= \reward(x,a)+\gamma\int_\X\max_{a'\in \A} Q(y,a')\dynamics(dy|x,a).
\end{equation*}

\paragTitle{Notation for function spaces} 
For any measure $\mu\in\mSpace(\XA)$ obtained from the combination of a distribution $\rho\in\mSpace(\X)$ and a uniform distribution over the discrete set $\A$, and a measurable function $f:\XA\rightarrow\Re$, we define the $L_2(\mu)$-norm of $f$ as $||f||^2_\mu=\avgActions\int_\X f(x,a)^2\rho(dx)$.
The supremum norm of $f$ is defined as $||f||_\infty=\sup_{x\in\X}|f(x)|$. Finally, we define the standard $L_{2}$-norm for a vector $\alpha\in\Re^{d}$ as $||\alpha||^{2} = \sum_{i=1}^{d} \alpha_{i}^{2}$. 
We denote by $\phi(\cdot,\cdot)=\big(\varphi_1(\cdot,\cdot),\ldots,\varphi_d(\cdot,\cdot)\big)^\top$ a feature vector with features $\varphi_i : \XA \rightarrow [-C,C]$, and by $\F=\{f_\alpha(\cdot,\cdot)=\phi(\cdot, \cdot)^\top\alpha\}$ the linear space of action-value functions spanned by the basis functions in $\phi$. Given a set of state-action pairs $\{(X_l,A_l)\}_{l=1}^\nTot$, let $\Phi=[\phi(X_1,A_1)^\top;\ldots;\phi(X_\nTot,A_\nTot)^\top]$ be the corresponding feature matrix. We define the orthogonal projection operator $\Pi:\B(\XA;V_{\max})\rightarrow \F$ as $\Pi Q = \arg\min_{f\in{\F}} ||Q-f||_{\mu}$. Finally, by $T(Q)$ we denote the truncation of a function $Q$ in the range $[-\Vmax,\Vmax]$.



\paragTitle{Problem setup} We consider the transfer problem in which $\nTasks$ tasks $\{\MDP_m\}_{m=1}^\nTasks$ are available and the objective is to learn the solution for the target task $\MDP_1$ transferring samples from the source tasks $\{\MDP_m\}_{m=2}^\nTasks$. We define an assumption on how the training sets are generated.

\begin{definition}\label{def:random.tasks}
\textbf{(Random Tasks Design)} An input set $\{(X_l,A_l)\}_{l=1}^{\nTot}$ is built with samples drawn from an arbitrary sampling distribution $\mu\in\mathcal S(\X\times\A)$, i.e. $(X_l,A_l) \sim \mu$. For each task $m$, one transition and reward sample is generated in each of the state-action pairs in the input set, i.e. $Y_l^m \sim \dynamics(\cdot|X_l,A_{l})$, and $R_l^m = \reward(X_l,A_l)$. Finally, we define the random sequence $\{M_{l}\}_{l=1}^{\nTot}$ where the indexes $M_{l}$ are drawn i.i.d. from a multinomial distribution with parameters $(\lambda_{1}, \ldots, \lambda_{\nTasks})$. The training set available to the learner is $\{(X_l,A_l,Y_l,R_l)\}_{l=1}^{\nTot}$ where $Y_{l} = Y_{l,M_{l}}$ and $R_{l} = R_{l,M_{l}}$.
\end{definition}

This is an assumption on how the samples are generated but in practice, a single realization of samples and task indexes $M_{l}$ is available.
We consider the case in which $\lambda_1\ll \lambda_m$ ($m=2,\ldots,\nTasks$). This condition implies that (on average) the number of target samples is much less than the source samples and it is usually not enough to learn an accurate solution for the target task. We will also consider the \textit{pure transfer} case in which $\lambda_1 = 0$ (i.e., no target sample is available).
Finally, we notice that Def.~\ref{def:random.tasks} implies the existence of a generative model for all the MDPs, since the state-action pairs are generated according to an arbitrary sampling distribution $\mu$.


\section{All-Sample Transfer Algorithm}\label{s:ast}

\begin{figure}[t]
\bookbox{
\begin{algorithmic}
\STATE \textbf{Input:} Linear space $\F = \text{span}\{\varphi_i, 1\leq i\leq d\}$, initial function $\tQ^{0}\in\F$
\STATE
\FOR{$k = 1,2,\ldots$}
\STATE Build the training set $\{(X_l,A_l,Y_l,R_l)\}_{l=1}^\nTot$ [according to \textit{random} tasks design]
\STATE Build the feature matrix $\Phi=[\phi(X_1,A_{1})^\top;\ldots;\phi(X_\nTot,A_{\nTot})^\top]$
\STATE Compute the vector $p\in\Re^{\nTot}$ with $p_{l} = R_l + \gamma\max_{a'\in\action}\tQ^{k-1}(Y_l,a')$
\STATE Compute the projection $\hat \alpha^{k} = (\Phi^\top \Phi)^{-1}\Phi^\top p$ and the function $\hQ^{k} = f_{{\halpha}^{k}}$
\STATE Return the truncated function $\tQ^{k} = \trunc(\hQ^{k})$
\ENDFOR
\end{algorithmic}}
\caption{A pseudo-code for All-Sample Transfer (AST) Fitted Q-iteration.}\label{f:all.transfer.alg}
\end{figure}

We first consider the case when the source samples are generated beforehand according to Def.~\ref{def:random.tasks} and the designer has no access to the source tasks. We study the algorithm called \textit{All-Sample Transfer} (AST) (Fig.~\ref{f:all.transfer.alg}) which simply runs FQI with a linear space $\F$ on the whole training set $\{(X_l,A_l,Y_l,R_l)\}_{l=1}^\nTot$. At each iteration $k$, given the result of the previous iteration $\tQ^{k-1} = \trunc(\hQ^{k-1})$, the algorithm returns
\begin{align}\label{e:all.transfer}
\hQ^{k} = \arg \min_{f\in\F} \avgTot \left(f(X_l, A_l) - (R_l + \gamma \max_{a'\in\A} \tQ^{k-1}(Y_l,a'))\right)^2.
\end{align}
In the case of linear spaces, the minimization problem is solved in closed form as in Fig.~\ref{f:all.transfer.alg}. In the following we report a finite-sample analysis of the performance of AST.
Similar to~\cite{munos2008finite}, we first study the prediction error in each iteration and we then propagate it through iterations.


\subsection{Single Iteration Finite-Sample Analysis}\label{ss:ast.random.bound.iter}

We define the \textit{average} MDP $\bMDP_\lambda$ as the average of the $\nTasks$ MDPs at hand. We define its reward function $\bReward_\lambda$ and its transition kernel $\bDynamics_\lambda$ as the weighted average of reward functions and transition kernels of the basic MDPs with weights determined by the proportions $\lambda$ of the multinomial distribution in the definition of the random tasks design (i.e., $\bReward_\lambda = \avgRatioB \reward_{m}$, $\bDynamics_\lambda = \avgRatioB\dynamics_{m}$). 
The resulting average Bellman operator is
\begin{align}\label{eq:avg.bellman}
(\bT_\lambda Q)(x,a) = \Big(\avgRatioB \T^{m} Q\Big)(x,a) = \bReward(x,a) + \gamma \int_{\X} \max_{a'} Q(y,a') \bDynamics(dy|x,a).
\end{align}
In the random tasks design, the average MDP plays a crucial role since the implicit target function of the minimization of the empirical loss in Eq.~\ref{e:all.transfer} is indeed $\bT_\lambda \tQ_{k-1}$. At each iteration $k$, we prove the following performance bound for AST.

\begin{theorem}\label{thm:all.transfer.rand.iter}
Let $\nTasks$ be the number of tasks $\{\MDP_m\}_{m=1}^\nTasks$, with $\MDP_1$ the target task. Let the training set $\{(X_l,A_l,Y_l,R_l)\}_{l=1}^{\nTot}$ be generated as in Def.~\ref{def:random.tasks}, with a proportion vector $\lambda = (\lambda_1,\ldots,\lambda_\nTasks)$. Let 
$f_{\alpha^k_*} = \Pi\T_1\tQ^{k-1} = \arg\inf_{f\in\F} ||f - \T_1 \tQ^{k-1}||_{\mu}$, then for any $0<\delta\leq 1$, $\hQ^k$ (Eq.~\ref{e:all.transfer}) satisfies
\begin{align}
||T(\hQ^k) &- \T_{1} \tQ^{k-1}||_\mu \leq 4 ||f_{\alpha^k_*} - \T_{1} \tQ^{k-1}||_{\mu} + 5\sqrt{\avgTaskDist_\lambda(\tQ^{k-1})} \nonumber\\
&+ 24(\Vmax + C||\alpha^k_*||)\sqrt{\frac{2}{\nTot}\log\frac{9}{\delta}}\nonumber + 32\Vmax \sqrt{\frac{2}{\nTot}\log \left(\frac{27(12\nTot e^2)^{2(d+1)}}{\delta}\right)}.
\end{align}
with probability $1-\delta$ (w.r.t. samples), where $||\varphi_i||_\infty\!\leq\! C$ and $\avgTaskDist_\lambda(\tQ^{k-1}) = \norm{(\T_1-\bT_\lambda)\tQ^{k-1}}_{\mu}^{2}$.
\end{theorem}

\textbf{Remark 1 (Analysis of the bound).} We first notice that the previous bound reduces (up to constants) to the standard bound for FQI when $M=1$ (see Section~\ref{app.fqi.linear}).
The bound is composed by three main terms: \textit{(i)} approximation error, \textit{(ii)} estimation error, and \textit{(iii)} transfer error. The approximation error $||f_{\alpha^k_*} - \T_1 \tQ^{k-1}||_{\mu}$ is the smallest error of functions in $\F$ in approximating the target function $\T_{1}\tQ^{k-1}$ and it is independent from the transfer algorithm. The estimation error (third and fourth terms in the bound) is due to the finite random samples used to learn $\hQ^k$ and it depends on the dimensionality $d$ of the function space and it decreases with the total number of samples $\nTot$ with the fast rate of linear spaces ($O(d/L)$ instead of $O(\sqrt{d/L})$). Finally, the transfer error $\avgTaskDist_\lambda$ accounts for the difference between source and target tasks. In fact, samples from source tasks different from the target might bias $\hQ^k$ towards a wrong solution, thus resulting in a poor approximation of the target function $\T_{1}\tQ^{k-1}$. It is interesting to notice that the transfer error depends on the difference between the target task and the average MDP $\bMDP_\lambda$ obtained by taking a linear combination of the source tasks weighted by the parameters $\lambda$. This means that even when each of the source tasks is very different from the target, if there exists a suitable combination which is similar to the target task, then the transfer process is still likely to be effective. Furthermore, $\avgTaskDist_\lambda$ considers the difference in the result of the application of the two Bellman operators to a given function $\tQ^{k-1}$. As a result, when the two operators $\T_1$ and $\bT_\lambda$ have the same reward functions, even if the transition distributions are different (e.g., the total variation $||\dynamics_1(\cdot|x,a)-\bDynamics_\lambda(\cdot|x,a)||_{\text{TV}}$ is large), their corresponding averages of $\tQ^{k-1}$ might still be similar (i.e., $\int \max_{a'} \tQ(y,a') \dynamics_1(dy|x,a)$ similar to $\int \max_{a'} \tQ(y,a')\bDynamics_\lambda(dy|x,a)$).

\textbf{Remark 2 (Comparison to single-task learning).} Let $\hQ^{k}_{s}$ be the solution obtained by solving one iteration of FQI with only samples from the source task, the performance bounds of $\hQ^k$ and $\hQ^k_{s}$ can be written as (up to constants and logarithmic factors)

\begin{align}
\norm{T(\hQ^k) - \T_1 \tQ^{k-1}}_{\mu} &\leq ||f_{\alpha^k_*} - \T_1 \tQ^{k-1}||_{\mu} + (\Vmax + C||\alpha_*^{k}||)\sqrt{\frac{1}{\nTot}} + \Vmax \sqrt{\frac{d}{\nTot}} + \sqrt{\avgTaskDist_\lambda} \;, \nonumber\\
\norm{T(\hQ^k_s) - \T_1 \tQ^{k-1}}_{\mu} &\leq ||f_{\alpha^k_*} - \T_1 \tQ^{k-1}||_{\mu} + (\Vmax + C||\alpha^k_*||)\sqrt{\frac{1}{\nStates_{1}}} + \Vmax \sqrt{\frac{d}{\nStates_{1}}}, \nonumber
\end{align}

with $\nStates_1=\lambda_1 L$ (on average). Both bounds share exactly the same approximation error. The main difference is that $\hQ^k_s$ uses only $\nStates_{1}$ samples and, as a result, has a much bigger estimation error than $\hQ^k$, which takes advantage of all the $\nTot$ samples transferred from the source tasks. At the same time, $\hQ^k$ suffers from an additional transfer error which does not exist in the case of $\hQ^k_s$. Thus, we can conclude that AST is expected to perform better than single-task learning whenever the advantage of using more samples is greater than the bias due to samples coming from tasks different from the target task. This introduces a \textit{transfer tradeoff} between including many source samples, so as to reduce the estimation error, and finding source tasks whose combination leads to a small transfer error. In Section~\ref{s:average.bound} 
we show how it is possible to define an adaptive transfer algorithm which selects proportions $\lambda$ so as to keep the transfer error $\avgTaskDist_\lambda$ as small as possible. Finally, in Section~\ref{s:tradeoff} we consider a different setting where the maximum number of samples in each source is fixed.



\subsection{Propagation Finite-Sample Analysis}\label{ss:ast.random.bound.prop}

We now study how the previous error is propagated through iterations. Let $\nu$ be the evaluation norm (i.e., in general different from the sampling distribution $\mu$). We first report two assumptions.~\footnote{We refer to~\cite{munos2008finite} for a thorough explanation of the concentrability terms.}

\begin{assumption}\label{a:concentrability} \textit{\cite{munos2008finite}}
Given $\mu$, $\nu$, $p\geq 1$, and an arbitrary sequence of policies $\{\pi_p\}_{p\geq 1}$, we assume that the future-state distribution $\mu\dynamics^1_{\pi_1}\cdots \dynamics^1_{\pi_p}$ is absolutely continuous w.r.t. $\nu$. We assume that $c(p) = \sup_{\pi_1\cdots \pi_p} || d(\mu\dynamics^1_{\pi_1}\cdots \dynamics^1_{\pi_p})/\nu ||_{\infty}$ satisfies $C_{\mu,\nu} = (1-\gamma^2)^2 \sum_{p} p\gamma^{p-1} c(p) < \infty$.
%
%
%
\end{assumption}

We also need the features $\varphi_i$ to be linearly independent w.r.t. $\mu$.

\begin{assumption}\label{a:linear.indep}
Let $G\in\Re^{d\times d}$ be the Gram matrix with $[G]_{ij} = \int \varphi_i(x,a) \varphi_j(x,a) \mu(dx,a)$. We assume that its smallest eigenvalue $\omega$ is strictly positive (i.e., $\omega>0$).
\end{assumption}

Using the two previous assumptions we derive the following performance bound for AST.


\begin{theorem}\label{thm:ast.random.bound.prop}
Let Assumptions~\ref{a:concentrability} and ~\ref{a:linear.indep} hold and the setting be as in Theorem~\ref{thm:all.transfer.rand.iter}. After $K$ iterations, AST returns an action-value function $\tQ_{K}$, whose corresponding greedy policy $\pi_{K}$ satisfies
\begin{align}
||Q^* &- Q^{\pi_{K}}||_{\nu} \leq \frac{2\gamma}{(1-\gamma)^{3/2}} \sqrt{C_{\mu,\nu}}\Bigg[ 4\sup_{g\in\F} \inf_{f\in\F}||f - \T_{1} g||_{\mu} + 5 \sup_{\alpha} \norm{(\T_1-\bT_\lambda)T(f_{\alpha})}_{\mu}\nonumber\\
& + 56(\Vmax+\frac{\Vmax}{\sqrt{\omega}})\sqrt{\frac{2}{\nTot}\log\frac{9K}{\delta}}\nonumber + 32\Vmax \sqrt{\frac{2}{\nTot}\log \left(\frac{27K(12\nTot e^2)^{2(d+1)}}{\delta}\right)} + \frac{2\Vmax}{\sqrt{C_{\mu,\nu}}}\gamma^K \Bigg].
\end{align}
\end{theorem}

\textbf{Remark (Analysis of the bound).} The bound reported in the previous theorem displays few differences w.r.t. to the single-iteration bound. The additional term $O(\gamma^K)$ accounts for the error due to the finite number of iterations of FQI and it decreases exponentially with base $\gamma$. The approximation error is now $\sup_{g} \inf_{f}||f - \T^{1} g||_{\mu}$. This term is referred to as the \textit{inherent Bellman error}~\cite{munos2008finite} of the space $\F$ and it is related to how well the Bellman images of functions in $\F$ can be approximated by $\F$ itself. It is possible to show that for particular classes of MDPs (e.g., Lipschitz), if a large enough number of carefully designed features is available, then this term is small. 
In the estimation error, the norm $||\alpha^k_*||$ is bounded using the linear independency between features (Assumption~\ref{a:linear.indep}) and the boundedness of the functions $\tQ^k$ returned at each iteration. The resulting term has an inverse dependency on the smallest eigenvalue $\omega$ which tends to be small whenever the Gram matrix is not well-defined (i.e., the features are almost linearly dependent). 
The transfer error $\sup_{\alpha} \norm{(\T_1-\bT_\lambda)T(f_{\alpha})}_{\mu}$ characterizes the difference between the target and average Bellman operators through the space $\F$. 
As a result, even MDPs with significantly different rewards and transitions might have a small transfer error because of the functions in $\F$. This introduces a tradeoff in the design of $\F$ between a ``large'' enough space containing functions able to approximate $\T_1 Q$ (i.e., small approximation error) and a small function space where the Q-functions induced by $\T_1$ and $\bT_\lambda$ can be closer (i.e., small transfer error). This term also displays interesting similarities with the notion of \textit{discrepancy} introduced in~\cite{mansour2009domain} in domain adaptation.


\section{Best Average Transfer Algorithm}\label{s:average.bound}

\begin{figure}[t]
\bookbox{
\begin{algorithmic}
\STATE \textbf{Input:} Space $\F = \text{span}\{\varphi_i, 1\leq i\leq d\}$, initial function $\tQ^{0}\in\F$, number of samples $L$
\STATE
\STATE Build the auxiliary set $\{(X_s,A_s,R_{s,1},\ldots,R_{s,\nTasks}\}_{s=1}^S$ and $\{Y_{s,1}^t,\!\ldots\!,Y_{s,\nTasks}^t\}_{t=1}^T$ for each $s$
\FOR{$k = 1,2,\ldots$}
\STATE Compute $\hlambda^k = \arg\min_{\lambda\in\Lambda} \hTaskDist_\lambda(\tQ^{k-1})$
\STATE Run one iteration of AST (Fig.~\ref{f:all.transfer.alg}) using $L$ samples generated according to $\hlambda^k$
\ENDFOR
\end{algorithmic}}
\caption{A pseudo-code for the Best Average Transfer (BAT) algorithm.}\label{f:average.alg}
\end{figure}

As discussed in the previous section, the transfer error $\avgTaskDist_\lambda$ plays a crucial role in the comparison with single-task learning. In particular, $\avgTaskDist_\lambda$ is related to the proportions $\lambda$ inducing the average Bellman operator $\bT_\lambda$ which defines the target function approximated at each iteration. 
We now consider the case where the designer has direct access to the source tasks (i.e., it is possible to choose how many samples to draw from each source) and can define an arbitrary proportion $\lambda$. In particular, we propose a method that adapts $\lambda$ at each iteration so as to minimize the transfer error $\avgTaskDist_\lambda$. 

We consider the case in which $L$ is fixed as a parameter of the algorithm and $\lambda_1=0$ (i.e., no target samples are used in the learning training set).
At each iteration $k$, we need to estimate the quantity $\avgTaskDist_\lambda(\tQ^{k-1})$.
We assume that for each task additional samples available. Let $\{(X_s,A_s,R_{s,1},\ldots,R_{s,\nTasks})\}_{s=1}^S$ be an \textit{auxiliary} training set where $(X_s,A_s)\sim\mu$ and $R_{s,m} = \reward_m(X_s,A_s)$. In each state-action pair, we generate $T$ next states for each task, that is $Y_{s,m}^t \sim \dynamics_m(\cdot | X_s,A_s)$ with $t=1,\ldots,T$. Thus, for any function $Q$ we define the estimated transfer error as

\begin{small}
\begin{align}\label{e:est.proportion}
\hTaskDist_\lambda(Q) \!=\! \avgDist \Bigg[ R_{s,1} \!-\! \avgRatioC R_{s,m} + \frac{\gamma}{T}\sum_{t=1}^T \Big(\max_{a'} Q(Y_{s,1}^t,a') \!-\! \avgRatioC\max_{a'} Q(Y_{s,m}^t,a')\Big) \Bigg]^2\!.
\end{align}
\end{small}

At each iteration, the algorithm \textit{Best Average Transfer} (BAT) (Fig.~\ref{f:average.alg}) first computes $\hlambda^k = \arg\min_{\lambda\in\Lambda} \hTaskDist_\lambda(\tQ^{k-1})$, where $\Lambda$ is the $(M\text{-}2)$-dimensional simplex, and then runs an iteration of AST with samples generated according to the proportions $\hlambda^k$. We denote by $\lambda^k_*$ the best combination at iteration $k$, that is
\begin{align}\label{e:opt.proportion}
\lambda_*^k = \arg\min_{\lambda \in \Lambda} \avgTaskDist_\lambda(\tQ^{k-1}) = \arg\min_{\lambda \in \Lambda} \expectB{\mu}{\Big(\sum_{m=2}^\nTasks \lambda_m (\T^m\tQ^{k-1})(x,a) - (\T^1\tQ^{k-1})(x,a)\Big)^2}.
\end{align}
The following performance guarantee can be proved for BAT.

\begin{lemma}\label{l:bat.rand.lambda}
Let $\{(X_s,A_s,R_s^1,\ldots,R_s^\nTasks)\}_{s=1}^S$ be a training set where $(X_s,A_s)\iid\mu$ and $R_s^m = \reward^m(X_s,A_s)$ and for each state-action pair and for each task $m$, $T$ next states $Y_{s,t}^m \sim \dynamics^m(\cdot | X_s,A_s)$ with $t=1,\ldots,T$ are available. For any fixed bounded function $Q\in\B(\state\times\action; \Vmax)$, the $\hlambda$ returned by minimizing Eq.~\ref{e:est.proportion} is such that 
\begin{align}\label{e:bat.rand.lambda}
\avgTaskDist_{\hlambda}(Q) - \avgTaskDist_{\lambda_*}(Q) \leq 2\Vmax\sqrt{\frac{(M-2) \log 4\nDist/\delta}{\nDist}} + 16\Vmax^2 \frac{\log 4SM / \delta}{T}
\end{align}
with probability $1-\delta$.
\end{lemma}

From the previous lemma the approximation performance of BAT at each iteration follows.

\begin{theorem}\label{thm:bat.rand.iter}
Let $\tQ^{k-1}$ be the function returned at the previous iteration and $\hQbat^k$ the function returned by the BAT algorithm (Fig.~\ref{f:average.alg}). Then for any $0<\delta\leq 1$,  $\hQbat^k$ satisfies 
\begin{align}
||T(\hQbat^k) &- \T_{1} \tQ^{k-1}||_\mu \leq 4 ||f_{\alpha^{k}_*} - \T_{1} \tQ^{k-1}||_{\mu} + 5\sqrt{\avgTaskDist_{\lambda^k_*}(\tQ^{k-1})} \nonumber\\
&+ 5\sqrt{2\Vmax}\Bigg(\frac{(M-2) \log 8\nDist/\delta}{\nDist}\Bigg)^{1/4} + 20 \Vmax\sqrt{ \frac{\log 8SM / \delta}{T}} \nonumber\\
&+ 24(\Vmax + C||\alpha^{k}_*||)\sqrt{\frac{2}{\nTot}\log\frac{18}{\delta}}\nonumber + 32\Vmax \sqrt{\frac{2}{\nTot}\log \left(\frac{54(12\nTot e^2)^{2(d+1)}}{\delta}\right)}.
\end{align}
with probability $1-\delta$.
\end{theorem}

\textbf{Remark 1 (Comparison with AST and single-task learning)}. The analysis of the bound shows that BAT outperforms AST whenever the advantage in achieving the smallest possible transfer error $\avgTaskDist_{\lambda^k_*}$ is larger than the additional estimation error due to the auxiliary training set. It is also interesting to compare BAT to single-task learning. In fact, BAT performs better than single-task learning whenever the best possible combination of source tasks has a small transfer error and the additional estimation error related to the auxiliary training set is smaller than the estimation error in single-task learning. In particular, this means that $O( (M/S)^{1/4} )+O( (1/T)^{1/2} )$ should be smaller than $O( (d/N)^{1/2} )$ (with $N$ the number of target samples). The number of calls to the generative model for BAT is $ST$. In order to have a fair comparison with single-task learning  we set $S=N^{2/3}$ and $T=N^{1/3}$, then we obtain the condition $M\leq d^2 N^{-4/3}$ that constrains the number of tasks to be smaller than the dimensionality of the function space $\F$. We remark that the dependency of the auxiliary estimation error on $M$ is due to the fact that the $\lambda$ vectors (over which the transfer error is optimized) belong to the simplex $\Lambda$ of dimensionality $M\text{-}2$. Hence, the previous condition suggests that, in general, adaptive transfer methods may significantly improve the transfer performance (i.e., in this case a smaller transfer error) at the cost of additional sources of errors which depend on the dimensionality of the search space used to adapt the transfer process (i.e., in this case $\Lambda$). 

\textbf{Remark 2 (Iterations).} BAT recomputes the proportions $\hlambda^k$ at each iteration $k$. In fact a combination $\lambda_1$ approximating well the reward function $\reward_1$ at the first iteration (i.e., $\reward_1 \approx \bReward_{\lambda^1}$) does not necessarily have a small transfer error $||(\T_1-\T_{\lambda^1})\tQ^1||_\mu$ at the second iteration. We further investigate how the best source combination changes through iterations in the experiments of Section~\ref{s:experiments}.

\textbf{Remark 3 (Best source combination).} The previous theorem shows that BAT achieves the smallest transfer error $\avgTaskDist_{\lambda_*^k}(\tQ^{k-1})$ at the cost of an additional estimation error which scales with the size of the auxiliary training set as $O( (M/S)^{1/4} )+O( (1/T)^{1/2} )$. 
We notice that the first term of the estimation error depends on how well the $\mu$ is approximated by using a finite number $S$ of state-action pairs and it has a slower rate w.r.t. the other terms. The second term depends on the number of next states $T$ simulated at each state-action pair which are used to estimate the value of the Bellman operators. As a result, in order to reduce the estimation error we need to increase both $S$ and the number of next states $T$ in each state-action pair. 
It is interesting to notice that similar estimation errors appear in FVI~\cite{munos2008finite} where the optimal Bellman operator is approximated by Monte-Carlo estimation. 

\textbf{Remark 4 (Training set).} The implicit assumption in the definition of the auxiliary training set is that it is possible to generate a series of next states and rewards for all the tasks at the same state-action pairs. If the source training sets are fixed in advance and the designer has no access to the source tasks, then this assumption is not verified and it is not possible to test the similarity between the MDP $\bMDP$ and the target task. Nonetheless, if the generative model for the source tasks is available at learning time, the auxiliary training set could be generated before the learning phase actually begins. Furthermore, in the theoretical analysis, BAT does not use the samples in the auxiliary training set at learning time. A trivial improvement is to include the auxiliary samples to the training set. 

\textbf{Remark 5 (Comparison to other transfer methods).} In~\cite{lazaric2008transfer} a method to compute the similarity between MDPs is proposed. In particular, the authors introduce the definition of \textit{compliance} as the average probability of the target samples to be generated from an sample-based estimation of the source MDPs. The compliance is later used to determine the proportion of samples to be transferred from each of the source tasks. Although this algorithm shares a similar objective as BAT, they use different notions of similarity. In particular, the method in~\cite{lazaric2008transfer} tries to identify source tasks which are \textit{individually} similar to the target task, while the transfer error minimized in BAT considers the \textit{average} MDP obtained by the transfer process. Furthermore, the notion of compliance tries to measures the overall distance between two MDPs, while $\taskDist_\lambda(Q)$ always measures the distance of the images of a function $Q$ through two different Bellman operators.

\textbf{Remark 6 (Computational complexity).}  Finally, we notice that the minimization of $\hTaskDist_\lambda$ is a convex quadratic problem since the objective function is convex in $\lambda$ and $\lambda$ belongs to the $(\nTasks\text{-}2)$-dimensional simplex.


\section{Best Transfer Trade-off Algorithm}\label{s:tradeoff}

\begin{figure}[t]
\bookbox{
\begin{algorithmic}
\STATE \textbf{Input:} Linear space $\F = \text{span}\{\varphi_i, 1\leq i\leq d\}$, initial function $\tQ^{0}\in\F$, maximum number of samples available for each task $\nStatesTask$, transfer parameter $c$
\STATE
\STATE Build a training set $\{X_s,A_s,R_s^1,\ldots,R_s^\nTasks\}_{s=1}^S$ and the next states $\{Y_{s,t}^1,\ldots,Y_{s,t}^\nTasks\}_{t=1}^T$ for each state-action pair
\STATE
\FOR{$k = 1,2,\ldots$}
\STATE Compute $\hbeta = \arg\min_{\beta\in [0,1]^\nTasks} \hTaskDist_\beta + c \sqrt{\frac{d}{\sum_{m=1}^\nTasks\beta_m \nStatesTask}}$
\STATE Run one iteration of AST (Fig.~\ref{f:all.transfer.alg}) using $L$ samples generated according to $\hbeta$
\ENDFOR
\end{algorithmic}}
\caption{A pseudo-code for Best Tradeoff Transfer (BTT).}\label{f:tradeoff.alg}
\end{figure}

The previous algorithm is proved to successfully estimate the combination of source tasks which better approximates the Bellman operator of the target task. Nonetheless, BAT relies on the implicit assumption that $L$ samples can always be generated from any source task~\footnote{If $\lambda_m=1$ for task $m$, then the algorithm would generate all the $L$ training samples from task $m$.} and it cannot be applied to the case where the number of source samples is limited. Here we consider the more challenging case where the designer has still access to the source tasks but only a limited number of samples is available in each of them. In this case, an adaptive transfer algorithm should solve a tradeoff between selecting as many samples as possible, so as to reduce the estimation error, and choosing the proportion of source samples properly, so as to control the transfer error. The solution of this tradeoff may return non-trivial results, where source tasks similar to the target task but with few samples are removed in favor of a pool of tasks whose average roughly approximate the target task but can provide a larger number of samples.


Here we introduce the \textit{Best Tradeoff Transfer} (BTT) algorithm (see Figure~\ref{f:tradeoff.alg}). Similar to BAT, it relies on an auxiliary training set to solve the tradeoff. We denote by $\nStatesTask$ the maximum number of samples available for source task $m$. Let $\beta\in [0,1]^{\nTasks}$ be a weight vector, where $\beta_m$ is the fraction of samples from task $m$ used in the transfer process. We denote by $\avgTaskDist_\beta$ ($\hTaskDist_\beta$) the transfer error (the estimated transfer error) with proportions $\lambda$ where $\lambda_m = (\beta_m \nStatesTask) / \sum_{m'} (\beta_{m'}N_{m'})$. At each iteration $k$, BTT returns the vector $\beta$ which optimizes the tradeoff between estimation and transfer errors, that is
\begin{align}\label{e:tradeoff}
\hbeta^k = \arg\min_{\beta\in [0,1]^\nTasks} \Big( \hTaskDist_\beta(\tQ^{k-1}) + \tau \sqrt{\frac{d}{\sum_{m=1}^\nTasks\beta_m \nStatesTask}} \Big),
\end{align}
where $\tau$ is a parameter. While the first term accounts for the transfer error induced by $\beta$, the second term is the estimation error due to the total amount of samples used by the algorithm.

Unlike AST and BAT, BTT is a heuristic algorithm motivated by the performance bound in Theorem~\ref{thm:all.transfer.rand.iter} and we do not provide any theoretical guarantee about its performance. The main technical difficulty w.r.t. the previous algorithms is that the setting considered here does not match the random task design assumption (see Def.~\ref{def:random.tasks}) since the number of source samples is constrained by $\nStatesTask$. As a result, given a proportion vector $\lambda$, we cannot assume samples to be drawn at random according to a multinomial of parameters $\lambda$. Without this assumption, it is an open question whether a similar bound to AST and BAT could be derived. Nonetheless, the experimental results reported in Section~\ref{s:experiments} show the effectiveness of BTT in solving the transfer tradeoff.



\section{Experiments}\label{s:experiments}

\begin{figure}[!t]
\begin{minipage}[t]{0.49\textwidth}
\captionof{table}{Parameters for the first set of tasks}
\label{T:tasks1}
\begin{scriptsize}
\begin{center}
\begin{tabular}{ccccc}
\toprule
tasks & $p$ & $l$ & $\eta$ & Reward\\
\otoprule
$\MDP_1$ & $0.9$ & $1$ & $0.1$ & $+1$ in $[-11,-9] \cup [9,11]$\\
\midrule
$\MDP_2$ & $0.9$ & $2$ & $0.1$ & $-5$ in $[-11,-9] \cup [9,11]$\\
$\MDP_3$ & $0.9$ & $1$ & $0.1$ & $+5$ in $[-11,-9] \cup [9,11]$\\
$\MDP_4$ & $0.9$ & $1$ & $0.1$ & $+1$ in $[-6,-4] \cup [4,6]$\\
$\MDP_5$ & $0.9$ & $1$ & $0.1$ & $-1$ in $[-6,-4] \cup [4,6]$\\
\bottomrule
\end{tabular}
\end{center}
\end{scriptsize} 

\end{minipage}
\hfill
\begin{minipage}[t]{0.49\textwidth}
\captionof{table}{Parameters for the second set of tasks}
\label{T:tasks2}
\begin{scriptsize}
\begin{center}
\begin{tabular}{ccccc}
\toprule
tasks & $p$ & $l$ & $\eta$ & Reward\\
\otoprule
$\MDP_1$ & $0.9$ & $1$ & $0.1$ & $+1$ in $[-11,-9] \cup [9,11]$\\
\midrule
$\MDP_6$ & $0.7$ & $1$ & $0.1$ & $+1$ in $[-11,-9] \cup [9,11]$\\
$\MDP_7$ & $0.1$ & $1$ & $0.1$ & $+1$ in $[-11,-9] \cup [9,11]$\\
$\MDP_8$ & $0.9$ & $1$ & $0.1$ & $-5$ in $[-11,-9] \cup [9,11]$\\
$\MDP_9$ & $0.7$ & $1$ & $0.5$ & $+5$ in $[-11,-9] \cup [9,11]$\\
\bottomrule
\end{tabular}
\end{center}
\end{scriptsize} 

\end{minipage}
\end{figure}

In this section, we report and discuss preliminary experimental results of the transfer algorithms introduced in the previous sections. The main objective is to illustrate the functioning of the algorithms and compare their results with the theoretical findings. Thus, we focus on a simple problem and we leave more challenging problems for future work.

We consider a continuous extension of the 50-state variant of the chain walk problem proposed in~\cite{lagoudakis2003least}. The state space is described by a continuous state variable $x$ and two actions are available: one that moves toward \emph{left} and the other toward \emph{right}. 
With probability $p$ each action makes a step of length $l$, affected by a noise $\eta$, in the intended direction, while with probability $1-p$ it moves in the opposite direction.
%
For the target task $\MDP_1$, the state--transition model is defined by the following parameters: $p=0.9$, $l=1$, and $\eta$ is uniform in the interval $[-0.1,0.1]$. The reward function provides $+1$ when the system state reaches the regions $[-11,-9]$ and $[9,11]$ and $0$ elsewhere.
Furthermore, to evaluate the performance of the transfer algorithms previously described, we considered eight source tasks $\{\MDP_2, \dots, \MDP_9\}$ whose state--transition model parameters and reward functions are reported in Tab.~\ref{T:tasks1} and~\ref{T:tasks2}.
To approximate the Q-functions, we use a linear combination of 20 radial basis functions. In particular, for each action, we consider $9$ Gaussians with means uniformly spread in the interval $[-20,20]$ and variance equal to $16$, plus a constant feature.
The number of iterations for the FQI algorithm has been empirically fixed to $13$.
Samples are collected through a sequence of episodes, each one starting from the state $x_0=0$ with actions chosen uniformly at random. For all the experiments, we average over $100$ runs and we report standard deviation error bars.

\begin{figure}[t]
\begin{center}
\includegraphics[width=0.495\textwidth]{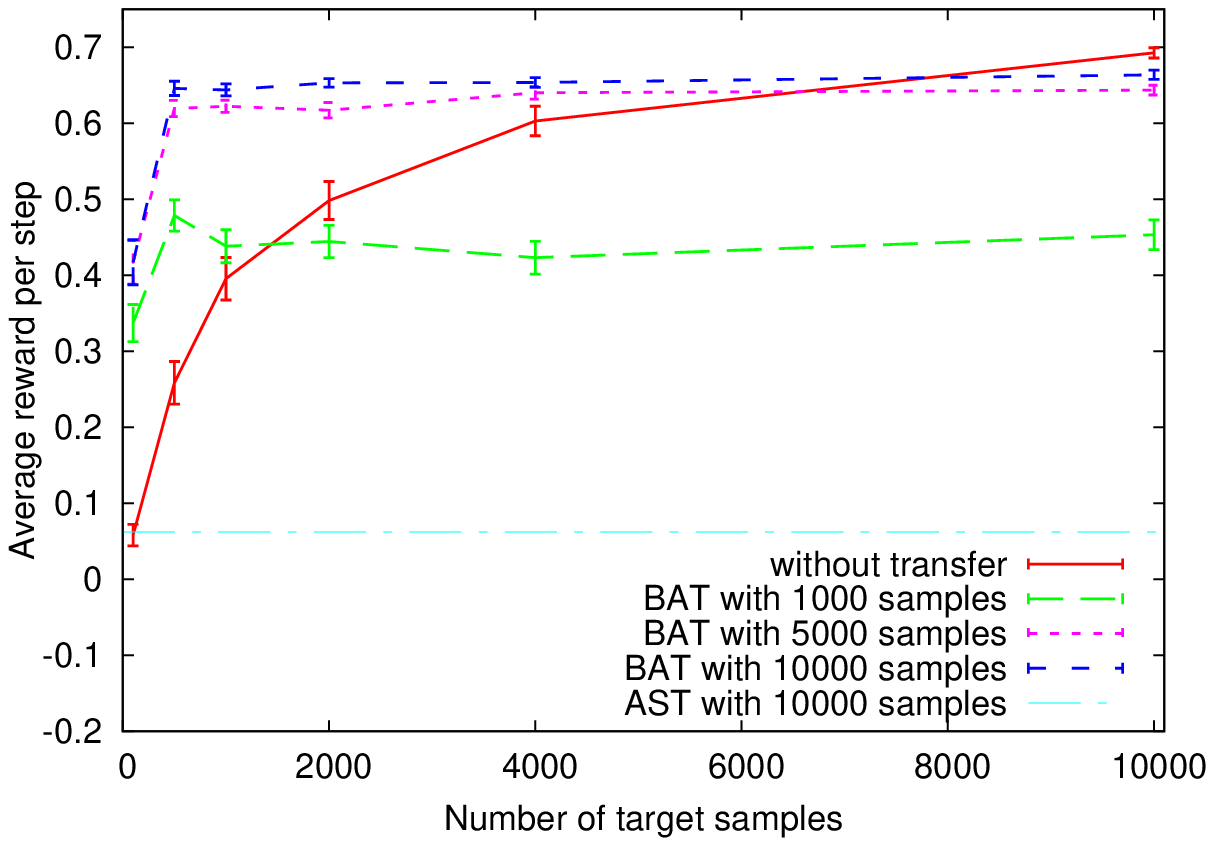}
\includegraphics[width=0.495\textwidth]{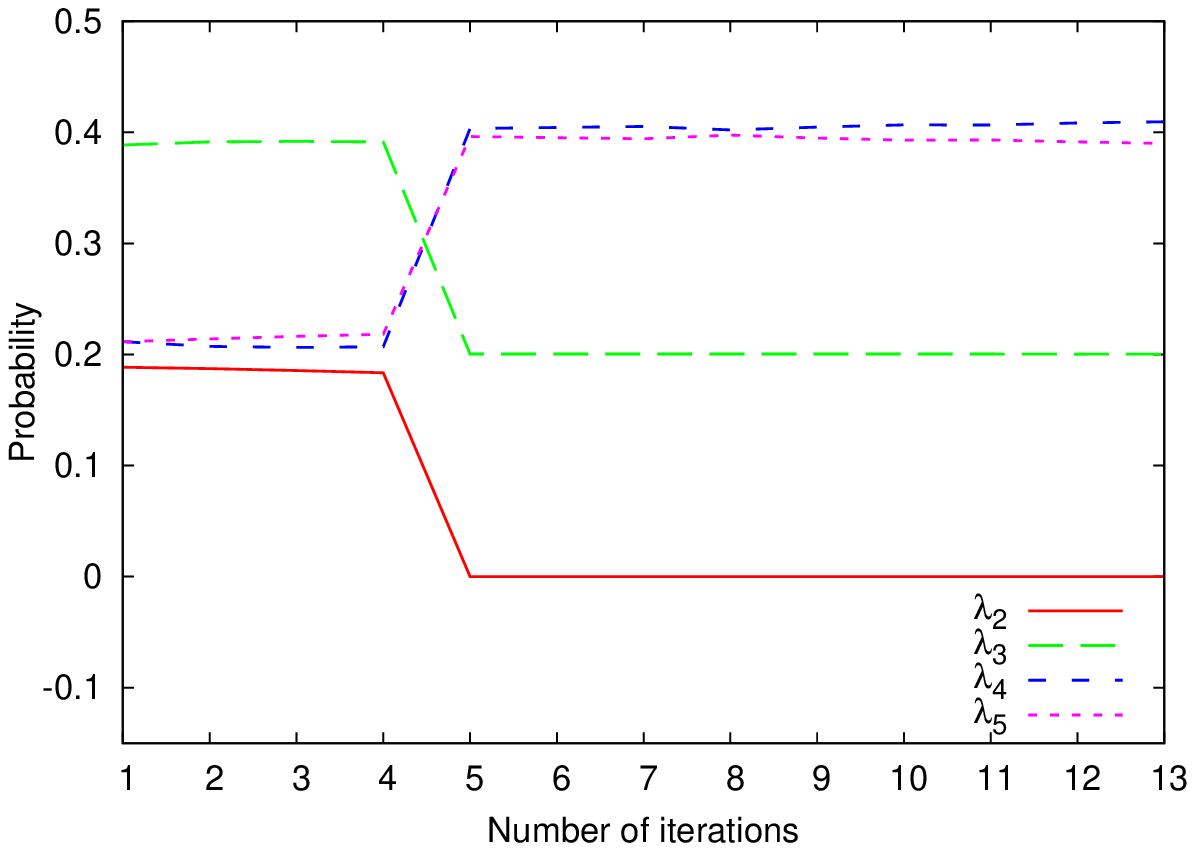}
\end{center}
\caption{Transfer from $\MDP_2$, $\MDP_3$, $\MDP_4$, $\MDP_5$. \emph{Left:} Comparison between single-task learning, AST with $L=10000$, BAT with $L=1000,5000,10000$. \emph{Right:} Source task probabilities estimated by BAT algorithm as a function of FQI iterations.}\label{F:transfer1}
\end{figure}

We first consider the \textit{pure} transfer problem where no target samples are actually used in the learning training set (i.e., $\lambda_1=0$). The objective is to study the impact of the transfer error due to the use of source samples and the effectiveness of BAT in finding a suitable combination of source tasks. 
The left plot in Fig.~\ref{F:transfer1} compares the performances of FQI with and without the transfer of samples from the first four tasks listed in Tab.~\ref{T:tasks1}. In case of single-task learning, the number of target samples refers to the samples used at learning time, while for BAT it represents the size $S$ of the auxiliary training set used to estimate the transfer error. Thus, while in single-task learning the performance increases with the target samples, in BAT they just make estimation of $\avgTaskDist_\lambda$ more accurate. The number of source samples added to the auxiliary set for each target sample was empirically fixed to one ($T=1$). We first run AST with $L=10000$ and $\lambda_2=\lambda_3=\lambda_4=\lambda_5=0.25$ (which on average corresponds to 2500 samples from each source). As it can be noticed by looking at the models in Tab.~\ref{T:tasks1}, this combination is very different from the target model and AST does not learn any good policy.
On the other hand, even with a small set of auxiliary target samples, BAT is able to learn good policies. Such result is due to the existence of linear combinations of source tasks which closely approximate the target task $\MDP_1$ at each iteration of FQI.
An example of the proportion coefficients computed at each iteration of BAT is shown in the right plot in Fig.~\ref{F:transfer1}. 
At the first iteration, FQI produces an approximation of the reward function. Given the first four source tasks, BAT finds a combination ($\lambda \simeq (0.2, 0.4, 0.2, 0.2)$) that produces the same reward function as $\reward_1$. However, after a few FQI iterations, such combination is no more able to accurately approximate functions $\T_1\tQ$. In fact, the state--transition model of task $\MDP_2$ is different from all the other ones (the step length is doubled). As a result, the coefficient $\lambda_2$ drops to zero, while a new combination among the other source tasks is found. Note that BAT significantly improves single-task learning, in particular when very few target samples are available.

In the general case, the target task cannot be obtained as any combination of the source tasks, as it happens by considering the second set of source tasks ($\MDP_6$, $\MDP_7$, $\MDP_8$, $\MDP_9$). The impact of such situation on the learning performance of BAT is shown in the left plot in Fig.~\ref{F:transfer2}. Note that, when a few target samples are available, the transfer of samples from a combination of the source tasks using the BAT algorithm is still beneficial. On the other hand, the performance attainable by BAT is bounded by the transfer error corresponding to the best source task combination (which in this case is large). As a result, single-task FQI quickly achieves a better performance. 

Results presented so far for the BAT transfer algorithm assume that FQI is trained only with the samples obtained through combinations of source tasks. Since a number of target samples is already available in the auxiliary training set, a trivial improvement is to include them in the training set together with the source samples (selected according to the proportions computed by BAT). As shown in the plot in the right side of Fig.~\ref{F:transfer2} this leads to a significant improvement. From the behavior of BAT it is clear that with a small set of target samples, it is better to transfer as many samples as possible from source tasks, while as the number of target samples increases, it is preferable to reduce the number of samples obtained from a combination of source tasks that actually does not match the target task. In fact, for $L=10000$, BAT has a much better performance at the beginning but it is then outperformed by single-task learning. On the other hand, for $L=1000$ the initial advantage is small but the performance remains close to single-task FQI for large number of target samples.
This experiment highlights the tradeoff between the need of samples to reduce the estimation error and the resulting transfer error when the target task cannot be expressed as a combination of source tasks (see Section~\ref{s:tradeoff}). BTT algorithm provides a principled way to address such tradeoff, and, as shown by the right plot in Fig.~\ref{F:transfer2}, it exploits the advantage of transferring source samples when a few target samples are available, and it reduces the weight of the source tasks (so as to avoid large transfer errors) when samples from the target task are enough.
It is interesting to notice that increasing the number of samples available for each source task from $5000$ to $10000$ improves the performance in the first part of the graph, while keeping unchanged the final performance. This is due to the capability of the BTT algorithm to avoid the transfer of source samples when there is no need for them, thus avoiding \emph{negative transfer} effects.


\begin{figure}[t]
\begin{center}
\includegraphics[width=0.495\textwidth]{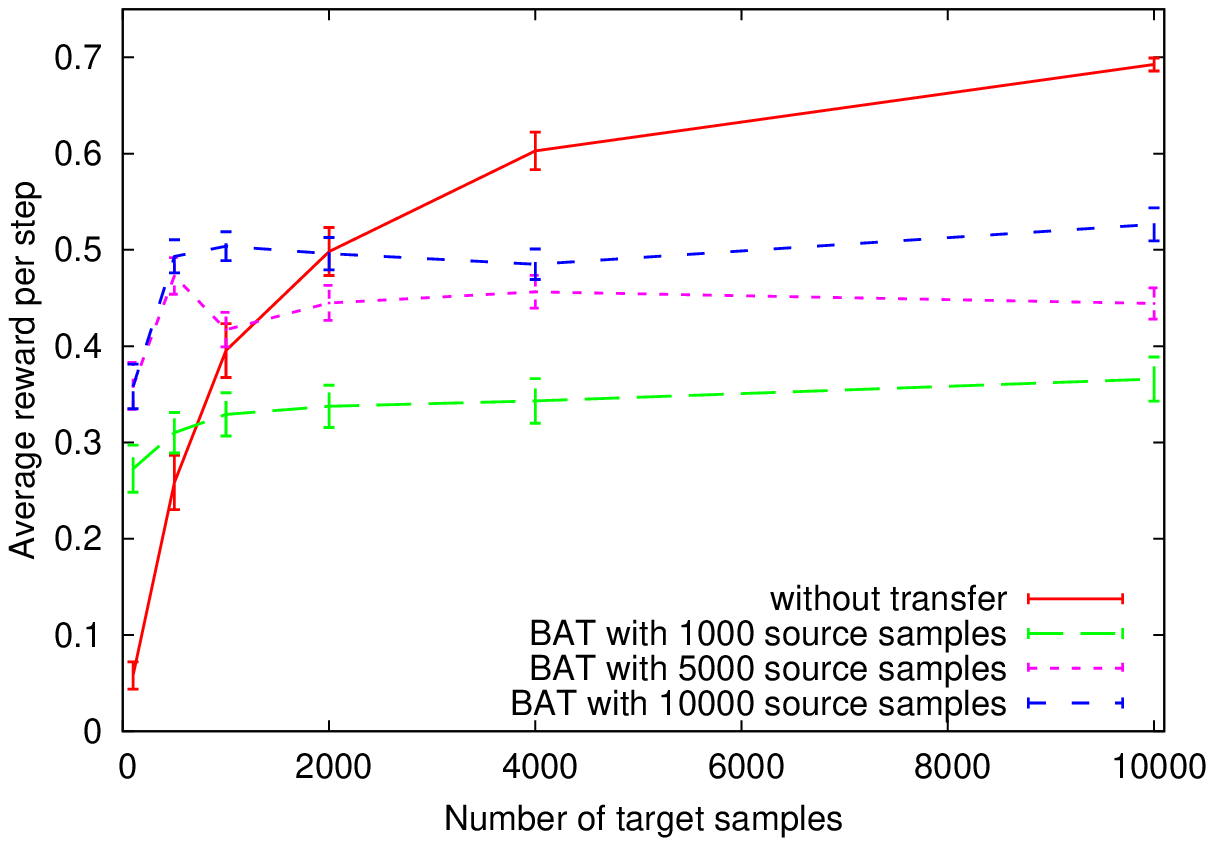}
\includegraphics[width=0.495\textwidth]{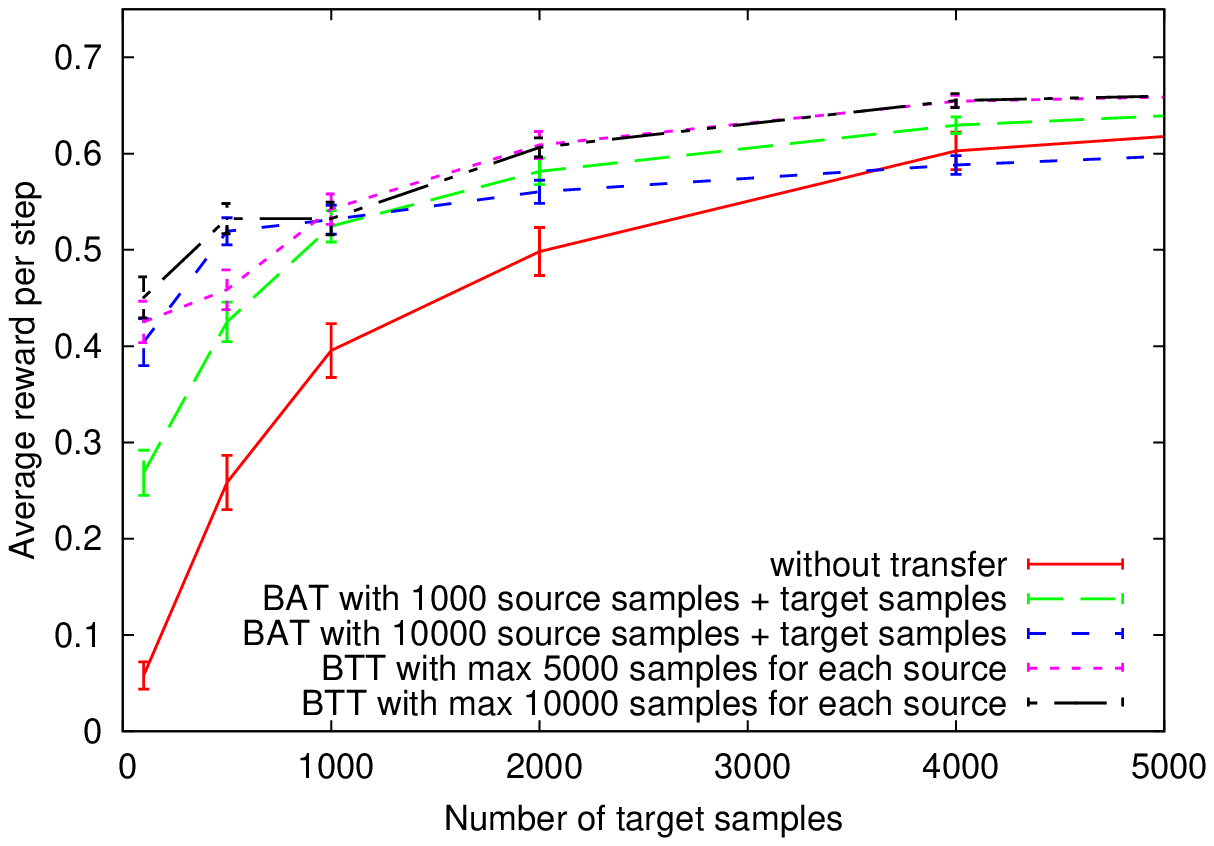}
\end{center}
\caption{Transfer from $\MDP_6$, $\MDP_7$, $\MDP_8$, $\MDP_9$. \emph{Left:} Comparison between single-task learning and BAT with $L=1000,5000,10000$. \emph{Right:} Comparison between single-task learning, BAT with $L=1000,10000$ in addition to the target samples, and BTT ($\tau = 0.75$) with $5000$ and $10000$ samples for each source task. To improve readability, the plot is truncated at 5000 target samples.}\label{F:transfer2}
\end{figure}


\section{Conclusions}\label{s:conclusions}

In this paper, we formalized and studied the sample-transfer problem. We first derived a finite-sample analysis of the performance of a simple transfer algorithm which includes all the source samples into the training set used to solve a given target task. At the best of our knowledge, this is the first theoretical result for a transfer algorithm in RL showing the potential benefit of transfer over single-task learning. Then, in the case when the designer has direct access to the source tasks, we introduced an adaptive algorithm which selects the proportion of source tasks so as to minimize the bias due to the use of source samples. Finally, we considered a more challenging setting where the number of samples available in each source task is limited and a tradeoff between the amount of transferred samples and the similarity between source and target tasks must be solved. For this setting, we proposed a principled adaptive algorithm. Finally, we report a detailed experimental analysis on a simple problem which confirms and supports the theoretical findings.

This work opens several directions for future work.
\begin{itemize}
\item \textit{Transfer with transformations.} In many problems, there exist simple transformations to the source tasks dynamics and reward which would increase their similarity w.r.t. the target task, thus making the transfer process more effective. How affine transformations could be used in the adaptive transfer algorithms presented in this paper is an interesting direction for future work. In particular, it is an open question whether the cost (in terms of samples) of finding a suitable transformation would be effectively counter-balanced by transferring more similar samples.
\item \textit{Transfer between tasks with different state-action spaces.} In many real applications source and target tasks might have a different number of state variables and different actions. Thus, the current work should be extended to the more general case of tasks with different state-action spaces and it should be integrated with inter-task mapping transfer methods (see \cite{taylor2009transfer}).
\item \textit{Transfer with fixed tasks design.} Definition~\ref{def:random.tasks} prescribes the process used to generate the training set used in learning the target task. At each state-action pair, the sample is generated from a source task chosen at random according to a multinomial distribution. When the designer has no access to the source tasks and their samples are generated beforehand, this generative model is not reasonable. A different model (\textit{fixed} tasks design) should be defined where each sample is coming from a specific source which is fixed in advance. An interesting direction for future work is to understand how this different generative model affects the performance of the transfer algorithm and whether it is possible to define effective adaptive algorithms for this case.
\end{itemize}

{\bf Acknowledgments} This work was supported by French National Research Agency through the projects EXPLO-RA $n^\circ$ ANR-08-COSI-004, by Ministry of Higher Education and Research, Nord-Pas de Calais Regional Council and FEDER through the ``contrat de projets {\'e}tat region 2007--2013", and by PASCAL2 European Network of Excellence. The research leading to these results has also received funding from the European Community's Seventh Framework Programme (FP7/2007-2013) under grant agreement n¡ 231495.


\newpage
\bibliographystyle{plain}
\bibliography{transfer}

\begin{thebibliography}{10}

\bibitem{antos2007fitted}
Andras Antos, R\'emi Munos, and Csaba Szepesvari.
\newblock {Fitted Q-iteration in continuous action-space MDPs}.
\newblock In {\em {Neural Information Processing Systems}}, Vancouver, Canada,
  2007.

\bibitem{ben-david2010a-theory}
Shai Ben-David, John Blitzer, Koby Crammer, Alex Kulesza, Fernando Pereira, and
  Jennifer Vaughan.
\newblock A theory of learning from different domains.
\newblock {\em Machine Learning}, 79:151--175, 2010.

\bibitem{crammer2008learning}
Koby Crammer, Michael Kearns, and Jennifer Wortman.
\newblock Learning from multiple sources.
\newblock {\em Journal of Machine Learning Research}, 9:1757--1774, 2008.

\bibitem{ernst2005tree-based}
Damien Ernst, Pierre Geurts, and Louis Wehenkel.
\newblock Tree-based batch mode reinforcement learning.
\newblock {\em J. Mach. Learn. Res.}, 6:503--556, December 2005.

\bibitem{GyKoKrWa02}
L.~Gy\"orfi, M.~Kohler, A.~Krzy\.zak, and H.~Walk.
\newblock {\em A distribution-free theory of nonparametric regression}.
\newblock Springer-Verlag, New York, 2002.

\bibitem{lagoudakis2003least}
M.G. Lagoudakis and R.~Parr.
\newblock Least-squares policy iteration.
\newblock {\em The Journal of Machine Learning Research}, 4:1107--1149, 2003.

\bibitem{Lazaric10FS}
A.~Lazaric, M.~Ghavamzadeh, and R.~Munos.
\newblock Finite-sample analysis of {LSTD}.
\newblock Technical Report inria-00482189, INRIA, 2010.

\bibitem{lazaric2008transfer}
A.~Lazaric, M.~Restelli, and A.~Bonarini.
\newblock Transfer of samples in batch reinforcement learning.
\newblock In {\em Proceedings of the Twenty-Fifth Annual International
  Conference on Machine Learning (ICML'08)}, pages 544--551, 2008.

\bibitem{lazaric2008knowledge}
Alessandro Lazaric.
\newblock {\em Knowledge Transfer in Reinforcement Learning}.
\newblock PhD thesis, Poltecnico di Milano, 2008.

\bibitem{mansour2009domain}
Yishay Mansour, Mehryar Mohri, and Afshin Rostamizadeh.
\newblock Domain adaptation: Learning bounds and algorithms.
\newblock In {\em Proceedings of the 22nd Conference on Learning Theory
  (COLT'09)}, 2009.

\bibitem{munos2008finite}
R.~Munos and Cs. Szepesv{\'a}ri.
\newblock Finite time bounds for fitted value iteration.
\newblock {\em Journal of Machine Learning Research}, 9:815--857, 2008.

\bibitem{sutton1998reinforcement}
Richard~S. Sutton and Andrew~G. Barto.
\newblock {\em Reinforcement Learning: An Introduction}.
\newblock MIT Press, Cambridge, MA, 1998.

\bibitem{taylor2008transferring}
Matthew~E. Taylor, Nicholas~K. Jong, and Peter Stone.
\newblock Transferring instances for model-based reinforcement learning.
\newblock In {\em Proceedings of the European Conference on Machine Learning
  (ECML'08)}, pages 488--505, 2008.

\bibitem{taylor2009transfer}
Matthew~E.\ Taylor and Peter Stone.
\newblock Transfer learning for reinforcement learning domains: A survey.
\newblock {\em Journal of Machine Learning Research}, 10(1):1633--1685, 2009.

\end{thebibliography}

\newpage
\appendix


\section{Additional Notation}\label{app:add.notation}

Besides the notation introduced in Section~\ref{s:preliminaries}, here we introduce additional symbols used in the proofs. We define two empirical norms on functions and vectors. Given a set of $\nStates$ state-action pairs $\{(X_n,A_n)\}_{n=1}^\nStates$ drawn i.i.d. from $\mu$ we define the empirical norm $||f||_{\hmu}$ as
\begin{equation*}\label{L2-emp-norm}
||f||^2_{\hmu}=\avgStates f(X_n,A_n)^2.
\end{equation*}
Similarly, given a vector $y\in \Re^{N}$ we define the empirical norm $||y||_{N}$ as
\begin{equation*}\label{L2-emp-norm-vect}
||y||^2_{N}=\avgStates y_{n}^2.
\end{equation*}
Given a set of $\nStates$ state-action pairs $\{(X_n,A_n)\}_{n=1}^\nStates$, let $\Phi=[\phi(X_1,A_1)^\top;\ldots;\phi(X_\nStates,A_\nStates)^\top]$ be the feature matrix defined at the states $\{(X_n,A_n)\}_{n=1}^\nStates$, and $\F_n=\{\Phi\alpha,\;\alpha\in\Re^d\}\subset\Re^\nStates$ be the corresponding vector space. We denote by $\hPi:\Re^\nStates\rightarrow\F_\nStates$ the empirical orthogonal projection onto $\F_\nStates$, defined by 
\begin{align}\label{e:emp.projection}
\hPi y=\argmin_{z\in\F_\nStates}||y-z||_\nStates.
\end{align}
Note that the orthogonal projection $\hPi y$ of any $y\in\Re^\nStates$ always exists and is unique. 


\section{Fitted Q-iteration with Linear Spaces}\label{app.fqi.linear}

\begin{figure}[t]
\bookbox{
\begin{algorithmic}
\STATE \textbf{Input:} Linear space $\F = \text{span}\{\varphi_i, 1\leq i\leq d\}$, initial function $\tQ^{0}$
\STATE
\FOR{$k = 1,2,\ldots$}
\STATE Draw training samples $\{(X_n,A_n,Y_n,R_n)\}_{n=1}^\nStates$
\STATE Build the feature matrix $\Phi=[\phi(X_1,A_{1})^\top;\ldots;\phi(X_n,A_{n})^\top]$
\STATE Compute the vector $p_{n} = R_n + \gamma\max_{a'\in\action}\tQ^{k-1}(Y_n,a')$
\STATE Compute the projection $\hat \alpha^{k} = (\Phi^\top \Phi)^{-1}\Phi^\top p$
\STATE Return the truncated function $\tQ^{k} = \trunc(f_{\hat \alpha_{k}})$
\ENDFOR
\end{algorithmic}}
\caption{A pseudo-code for Fitted Q-iteration.}\label{f:fqi.algorithm}
\end{figure}

Although fitted iterative methods have been already analyzed in detail in~\cite{munos2008finite} and \cite{antos2007fitted}, at the best of our knowledge no explicit finite-sample bounds for FQI with linear spaces is available. Since at each iteration, FQI solves an explicit regression problem, the derivation is mostly a straightforward application of regression bounds for linear spaces and quadratic loss. Here we just report the result and the proof of the single iteration error for the so-called fixed and random samples design settings.

In Algorithm~\ref{f:fqi.algorithm} we report the structure of the algorithm.


\subsection{Fixed Samples Design}\label{app:fqi.fixed.design}

 \begin{figure}[ht!]
 \centering
 \includegraphics[width=0.4\columnwidth]{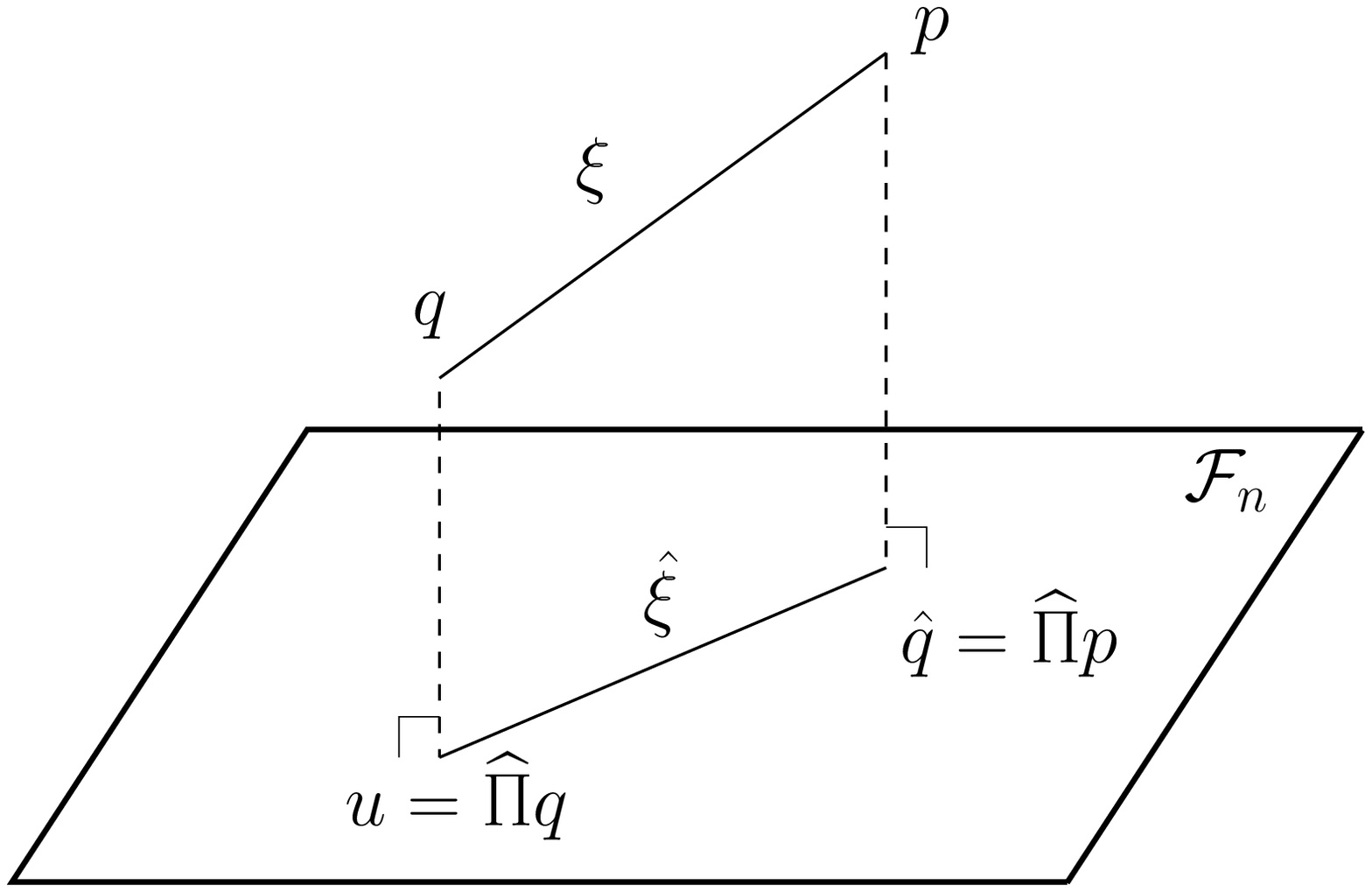}
 \vspace{-0.1in}
 \caption{This figure shows that the vectors used in the proof of Theorem~\ref{thm:fqi.emp.bound}.}
 \label{f:projection}
 \end{figure}

Similar to the analysis of LSTD in~\cite{Lazaric10FS} we first derive the fixed design bound (i.e., the performance is evaluated exactly on the states in the training set).

\begin{theorem}\label{thm:fqi.emp.bound}
Let $\F=\{\phi(\cdot,\cdot)^\top\alpha, \alpha\in\Re^d\}$ be a $d$-dimensional linear space. Let $\{(x_n,a_n,Y_n,R_n)\}_{n=1}^\nStates$ be the training set where $\{(x_n,a_n)\}_{n=1}^\nStates$ is an arbitrary sequence of state-action pairs, $Y_n \sim \dynamics(\cdot|x_n,a_n)$, and $R_n = \reward(x_n,a_n)$. Given a function $Q\in\mathcal B(\XA,\Vmax)$, let $q\in\Re^\nStates$ be the vector whose components are $q_n=(\T Q)(x_n,a_n)$ and $\hat q$ be the solution of a single iteration of fitted value iteration. Then with probability $1-\delta$ (w.r.t. the random next states $Y_n$), $\hat q$ satisfies
\begin{align}\label{eq:fqi.emp.bound}
||\hat q - q||_\nStates \leq ||\hPi q - q||_\nStates + 4\Vmax \sqrt{\frac{2}{\nStates}\log \left(\frac{3(9\nStates e^2)^{d+1}}{\delta}\right)}.
\end{align}
\end{theorem}

\begin{proof}
We denote by $u\in\Re^\nStates$ the orthogonal projection of the target vector $q$ onto the vector space $\F_\nStates$, that is $u = \hPi q$. By the definition of orthogonal projection and the Pythagorean theorem we decompose the error $||\hat q - q||_\nStates$ as
\begin{align}
||\hat q - q||_\nStates^2 = ||\hat q - u||_\nStates^2 + ||u - q||_\nStates^2,
\end{align}
where the first term represents the estimation error and the second term is the approximation error (see Fig.~\ref{f:projection}). We denote by $\noise_n = p_n - q_n$ the noise in the observations $p$ w.r.t. $q$. It is easy to notice that
\begin{align}
\expectA{\noise_n} = \expectB{Y\sim\dynamics(\cdot|x_n,a_n)}{\reward(x_n,a_n,Y) + \gamma\max_{a'\in\action} Q(Y, a')} - (\T Q)(x_n,a_n) = 0,
\end{align}
and that $|\noise_n| \leq 2\Vmax$. We also define the projected noise $\hnoise_n = \hat q_n - u_n$, that is $\hnoise = \hPi \noise$. Thus, we can rewrite the estimation error as
\begin{align}\label{eq:est-error1}
||\hat q - u||_\nStates^2 = ||\hnoise||_\nStates^2 = \langle \hnoise, \hnoise \rangle = \langle \noise, \hnoise \rangle,
\end{align}
where the last equality follows from the fact that $\hnoise$ is the orthogonal projection of $\noise$. Since $\hnoise \in \F_\nStates$, let $f_\beta\in\F$ be any function such that $f_\beta(x_n,a_n) = \hnoise_n$, and by a straightforward application of a variation of Pollard's inequality~\cite{GyKoKrWa02} we obtain
\begin{align}
\langle \noise, \hnoise \rangle &= \avgStates \noise_n f_\beta(x_n,a_n) \leq 4\Vmax \left(\avgStates f_\beta(x_n,a_n)^2\right)^{1/2} \sqrt{\frac{2}{\nStates}\log \left(\frac{3(9\nStates e^2)^{d+1}}{\delta}\right)} \nonumber\\
&= 4\Vmax ||\hnoise||_\nStates \sqrt{\frac{2}{\nStates}\log \left(\frac{3(9\nStates e^2)^{d+1}}{\delta}\right)}
\end{align}
with probability $1-\delta$. Thus from equation~\ref{eq:est-error1} we bound the estimation error by
\begin{align}
||\hat q - u||_\nStates \leq 4\Vmax \sqrt{\frac{2}{\nStates}\log \left(\frac{3(9\nStates e^2)^{d+1}}{\delta}\right)}.
\end{align}
Putting together the estimation error bound and the approximation error term, the statement of the theorem follows.
\end{proof}


\subsection{Random Samples Design}\label{app:fqi.random.design}

While in the previous section we analyzed the performance of FQI on the very same state-action pairs in the training set, we now focus on the generalization (i.e., prediction) performance on the whole state-action space.

Let $\hat Q$ be any function $f_\halpha\in\F$ satisfying $\Phi\halpha=\hat q$, where $\hat q$ is the vector defined in the previous section. Then we derive the following theorem. 

\begin{theorem}\label{thm:fqi.gen.bound}
Let $\F=\{\phi(\cdot,\cdot)^\top\alpha, \alpha\in\Re^d\}$ be a $d$-dimensional linear space. Let $\{(X_n,A_n,Y_n,R_n)\}_{n=1}^\nStates$ be the training set where $(X_n,A_n)\iid\mu$, $Y_n \sim \dynamics(\cdot|X_n,A_n)$, and $R_n = \reward(X_n,A_n)$. Given a function $Q\in\mathcal B(\XA,\Vmax)$, let $\hQ$ be the solution of a single iteration of fitted value iteration. Then with probability $1-\delta$ (w.r.t. the samples and the next states), $\hQ$ satisfies
\begin{align}\label{eq:fqi.gen.bound}
||T(\hQ) - \T Q||_\mu &\leq 4 \inf_{f_{\alpha}\in\F} ||f_{\alpha} - \T Q||_{\mu} \nonumber\\
&+ 24(\Vmax + L||\alpha_*||)\sqrt{\frac{2}{\nStates}\log\frac{9}{\delta}}\nonumber\\
&+ 32\Vmax \sqrt{\frac{2}{\nStates}\log \left(\frac{27(12\nStates e^2)^{2(d+1)}}{\delta}\right)}.
\end{align}
\end{theorem}

\begin{proof}
The proof mainly relies on the application of concentration of measures inequalities for linear spaces to the deterministic design bound in Theorem~\ref{thm:fqi.emp.bound}.

Let $f_{\halpha_*}\in\F$ be any function such that $f_{\halpha_*}(X_n) = (\hPi q)_n$, thus the approximation error $||\hPi q - q||_\nStates$ can be rewritten as $||f_{\halpha_*} - \T Q||_{\hmu}$. Furthermore we denote by $f_{\alpha_*} = \Pi(\T Q)$, that is the best approximation of the target function $\T Q$ onto $\F$ w.r.t. the distribution $\mu$. Since $f_{\halpha_*}$ is the minimizer of the empirical squared error, any function in $\F$ different from $f_{\halpha_*}$ has a bigger empirical loss, thus we obtain
\begin{align}\label{eq:approx.error}
||f_{\halpha_*} - \T Q||_{\hmu} \leq ||f_{\alpha_*} - \T Q||_{\hmu} \leq 2 ||f_{\alpha_*} - \T Q||_{\mu} + 12(\Vmax + L||\alpha_*||)\sqrt{\frac{2}{\nStates}\log\frac{3}{\delta'}},
\end{align}
with probability $1-\delta'$, where the second inequality is an application of a variation of Theorem~11.2 in~\cite{GyKoKrWa02} with a bound $||f_{\alpha_*} - \T Q||_{\infty} \leq \Vmax + L||\alpha_*||$. 
Similar, we notice that the left hand side of Eq.~\ref{eq:fqi.emp.bound} is $||\hat q - q||_{N} = ||\hQ - \T^{*}Q||_{\hmu}$ and we obtain
\begin{align}\label{eq:left.bound}
2||\hQ - \T Q||_{\hmu} \geq 2||T(\hQ) - \T Q||_{\hmu} \geq ||T(\hQ) - \T Q||_{\mu} - 24 \Vmax\sqrt{\frac{2}{\nStates}\log\left(\frac{9(12e\nStates)^{2(d+1)}}{\delta'}\right)}
\end{align}
with probability $1-\delta'$, where the second inequality is an application of a variation of Theorem~11.2 in~\cite{GyKoKrWa02}. Putting together Eqs~\ref{eq:fqi.emp.bound}, \ref{eq:approx.error}, and \ref{eq:left.bound} we obtain
\begin{align*}
||T(\hQ) - \T Q||_{\mu} \leq &2\Bigg(2 ||f_{\alpha_*} - \T Q||_{\mu} + 12(\Vmax + L||\alpha_*||)\sqrt{\frac{2}{\nStates}\log\frac{3}{\delta'}}+ \\
&+ 4\Vmax \sqrt{\frac{2}{\nStates}\log \left(\frac{3(9\nStates e^2)^{d+1}}{\delta'}\right)} \Bigg) 
+ 24 \Vmax\sqrt{\frac{2}{\nStates}\log\left(\frac{9(12e\nStates)^{2(d+1)}}{\delta'}\right)}
\end{align*}
Finally, by setting $\delta = 3\delta'$ the statement follows.
\end{proof}


\section{Analysis of AST}\label{app:analysis.ast}


\subsection{Proof of Theorem~\ref{thm:all.transfer.rand.iter}}\label{app:proof.ast.iter}

\begin{proof}

Since the proof follows similar steps as in the proof of Theorem~\ref{thm:fqi.gen.bound}, we discuss here only the fixed samples design bound. We define the vector $p\in\Re^{\nTot}$ such that for any $l=1,\ldots,\nTot$, $p_{l} = \sumTasks \ind{M_{l}=m} (R_l^{m} + \gamma\max_{a'} Q(Y_l^{m},a'))$. The target vector $q\in\Re^{\nTot}$ is the image of the function $Q$ through the average optimal Bellman operator. In fact, by defining $q_{l} = (\bT_\lambda Q)(X_l,A_l)$ we obtain a zero-mean noise vector $\noise_{l} = p_{l} - q_{l}$ such that $\expectA{\noise_{l}} = 0$ and $|\noise_{l}| \leq 2\Vmax$.~\footnote{The expectation is taken w.r.t. both the random realization of the reward $R_{l}^{m}$ and next state $Y_{l}^{m}$ and task index $M_{l}$.} 

The statement of the theorem simply follows by decomposing the prediction error of $\hQall$ as
\begin{align}
\norm{T(\hQall) - \T_1 Q}_{\mu} \leq \norm{T(\hQall) - \bT_\lambda Q}_{\mu} + \norm{\bT_\lambda Q - \T_1 Q}_{\mu}.
\end{align}
By substituting $\norm{T(\hQall) - \bT_\lambda Q}_{\mu}$ with a FQI bound w.r.t. the target function $\bT_\lambda Q$ we obtain
\begin{align}
\norm{T(\hQall) - \T_1 Q}_{\mu} &\leq  4 ||f_{\alpha} - \bT_\lambda Q||_{\mu} + \norm{\bT_\lambda Q - \T_1 Q}_{\mu}\\
&+ 24(\Vmax + C||\alpha||)\sqrt{\frac{2}{\nTot}\log\frac{9}{\delta}}\nonumber\\
&+ 32\Vmax \sqrt{\frac{2}{\nTot}\log \left(\frac{27(12\nTot e^2)^{2(d+1)}}{\delta}\right)}.
\end{align}
By rewriting the approximation error as $||f_{\alpha} - \bT_\lambda Q||_{\mu} \leq ||f_{\alpha} - \T^{1} Q||_{\mu} + ||\T^{1}Q - \bT_\lambda Q||_{\mu}$ and using $\alpha = \alpha_{*}$ the final bound follows.

\end{proof}


\subsection{Proof of Theorem~\ref{thm:ast.random.bound.prop}}\label{app:proof.ast.prop}

\begin{proof} \textit{[Sketch]}
The main structure of the proof is exactly the same as in~\cite{munos2008finite}. The main differences are due to the use of linear spaces and the transfer error. Following the passages in the proof of Theorem 2 in~\cite{munos2008finite}, we obtain
\begin{align*}
||Q^* - Q^{\pi_{K}}||_{\nu} \leq \frac{2\gamma}{(1-\gamma)^{3/2}} \Bigg[ \sqrt{C_{\mu,\nu}}\max_{k} ||T(\hQ^k) &- \T^{1} \tQ^k||_\mu + 2\Vmax\gamma^K \Bigg].
\end{align*}
Thus, we need to study all the terms in the statement of Theorem~\ref{thm:all.transfer.rand.iter} affected by the maximization over the iterations. 

\textit{Approximation error.} The approximation term becomes
\begin{align*}
\max_{k} \min_{f\in\F}||f - \T^{1} \tQ^k||_{\mu} \leq \sup_{g\in\F} \min_{f\in\F}||f - \T^{1} g||_{\mu}.
\end{align*}
This term is referred to as the inherent Bellman error of the space $\F$ and it is related to how well the Bellman images of functions in $\F$ can be approximated by $\F$ itself. 

\textit{Estimation error.} The second relevant term is the term $||\alpha^k_*||$ appearing in the estimation error. We recall that $f_{\alpha_*^k} = \Pi \T_1 \tQ^{k-1}$ is the projection on $\F$ of the Bellman image of the function returned at the previous iteration. The function $\tQ^{k-1}$ is truncated in the interval $[-\Vmax,\Vmax]$ and its Bellman image $\T_1 \tQ^{k-1}$ is still bounded in the same interval. Since the projection operator $\Pi$ is a non-expansion, we finally have that $||f_{\alpha_*^k}||_{\infty} \leq \Vmax$. Using Assumption~\ref{a:linear.indep}, for any $f_\alpha\in\F$, it is possible to relate the norm of the function to the norm of the vector $\alpha$ as
\begin{align*}
||f_{\alpha}||_{\mu}^2 = ||\phi^\top\alpha||_{\mu}^2 = \alpha^\top G \alpha \geq \omega \alpha^\top \alpha = \omega||\alpha||^2.
\end{align*}
By combining the bound on $\alpha$ with the bound on $f_{\alpha}$, we obtain that
\begin{align*}
\max_k||\alpha^k_*|| \leq \max_k\frac{||f_{\alpha^k_*}||_{\mu}}{\sqrt{\omega}} \leq \frac{\Vmax}{\sqrt{\omega}}
\end{align*}

\textit{Transfer error.} Since $\tQ^k$ is the truncation of a function $f_{\halpha^k}=\hQ^k$ belonging to $\F$, the transfer error is
\begin{align*}
\max_k ||(\T_1-\bT_\lambda)\tQ^k||_{\mu} = \sup_{\alpha} \norm{(\T_1-\bT_\lambda)T(f_{\alpha})}_{\mu}.
\end{align*}

Finally, the statement of the theorem follows by taking a union bound over $K$ iterations.

\end{proof}


\section{Analysis of BAT}\label{app.proof.adapt}


\subsection{Proof of Theorem~\ref{thm:bat.rand.iter}}

\begin{lemma}\label{l:bat.rand.lambda}
Let $\{(X_s,A_s,R_s^1,\ldots,R_s^\nTasks)\}_{s=1}^S$ be a training set where $(X_s,A_s)\iid\mu$ and $R_s^m = \reward^m(X_s,A_s)$ and for each state-action pair and for each task $m$, $T$ next states $Y_{s,t}^m \sim \dynamics^m(\cdot | X_s,A_s)$ with $t=1,\ldots,T$ are available. For any fixed bounded function $Q\in\B(\state\times\action; \Vmax)$, the $\hlambda$ returned by minimizing Eq.~\ref{e:est.proportion} is such that 
\begin{align}\label{e:bat.rand.lambda}
\avgTaskDist_{\hlambda}(Q) - \avgTaskDist_{\lambda_*}(Q) \leq 2\Vmax\sqrt{\frac{(M-2) \log 4\nDist/\delta}{\nDist}} + 16\Vmax^2 \frac{\log 4SM / \delta}{T}
\end{align}
with probability $1-\delta$.
\end{lemma}

\begin{proof}\textit{[Lemma~\ref{l:bat.rand.lambda}]}

The sketch of the proof is as follows. For any state-action pair $X_s, A_s$, we define
\begin{align*}
\hTaskDist_\lambda(X_s,A_s) = R_s^1 - \avgRatioC R_s^m + \gamma\frac{1}{T}\sum_{t=1}^T \Big(\max_{a'} \tQ^{k-1}(Y_{s,t}^1,a')-\avgRatioC\max_{a'} \tQ^{k-1}(Y_{s,t}^m,a')\Big),
\end{align*}
and
\begin{align*}
\taskDist_\lambda(X_s,A_s) = (\T_1\tQ^{k-1})(X_s,A_s) - \avgRatioC (\T^m\tQ^{k-1})(X_s,A_s).
\end{align*}
As a result, $\avgTaskDist_\lambda = \mathbb E_\mu \big[\taskDist_\lambda(x,a)^2\big]$ and $\hTaskDist_\lambda = \avgDist\hTaskDist_\lambda(X_s,A_s)^2$. By Pollard's inequality on the $(M\text{-}2)$-dimensional simplex $\Lambda$, we have for any $\lambda\in\Lambda$
\begin{align}\label{e:lemma1.step1}
| E_\mu \big[\taskDist_\lambda(x,a)^2\big] - \avgDist\taskDist_\lambda(X_s,A_s)^2 | \leq \Vmax\sqrt{\frac{(M-2) \log \nDist/\delta'}{\nDist}}
\end{align}
with probability $1-\delta'$. Using Chernoff-Hoeffding inequality we now bound the distance between the true Bellman operators in $\taskDist_\lambda(X_s,A_s)$ and their estimates in $\hTaskDist_\lambda(X_s,A_s)$. By triangle inequality and the previous definitions, we obtain the following series of inequlities
\begin{align}\label{e:lemma1.step2}
| \avgDist&\taskDist_\lambda(X_s,A_s)^2 - \avgDist\hTaskDist_\lambda(X_s,A_s)^2 | \leq | \avgDist(\taskDist_\lambda(X_s,A_s) - \hTaskDist_\lambda(X_s,A_s))^2 | \nonumber\\
&\leq \max_{s} \Big( \taskDist_\lambda(X_s,A_s)- \avgDist\hTaskDist_\lambda(X_s,A_s)\Big)^2 \nonumber\\
&\leq 2\max_s \max_m \Big((\T^m\tQ^{k-1})(X_s,A_x) - R_s^m - \gamma \frac{1}{T} \sum_{t=1}^T \max_{a'} Q(Y_{s,t}^m,a') \Big)^2\nonumber\\
&\leq 2\Bigg( 2\Vmax \sqrt{\frac{\log SM / \delta'}{T}} \Bigg)^2
\end{align}
By using Eqs \ref{e:lemma1.step1} and \ref{e:lemma1.step2}, we have for any $\lambda\in\Lambda$
\begin{align*}
|\avgTaskDist_\lambda - \hTaskDist_\lambda| \leq \Vmax\sqrt{\frac{(M-2) \log \nDist/\delta'}{\nDist}} + 8\Vmax^2 \frac{\log SM / \delta'}{T},
\end{align*}
with probability $1-2\delta'$. Finally, we can prove the following sequence of inequalities
\begin{align*}
\avgTaskDist_{\hlambda} - \avgTaskDist_{\lambda_*} &= \avgTaskDist_{\hlambda} - \hTaskDist_{\hlambda} + \hTaskDist_{\hlambda} - \hTaskDist_{\lambda_*} + \hTaskDist_{\lambda_*} - \avgTaskDist_{\lambda_*} \\
&\leq 2 \sup_{\lambda\in\Lambda} |\avgTaskDist_\lambda - \hTaskDist_\lambda| \leq 2\Vmax\sqrt{\frac{(M-2) \log \nDist/\delta'}{\nDist}} + 16\Vmax^2 \frac{\log SM / \delta'}{T},
\end{align*}
with probability $1-4\delta'$. By setting $\delta = 4\delta'$ the statement follows.
\end{proof}


\section{Additional Experimental Analysis}

In this section, we provide additional experimental results related to the BTT algorithm.

\subsection{Analysis of parameters $\beta$}

\begin{figure}[t]
\begin{center}
\includegraphics[width=0.495\textwidth]{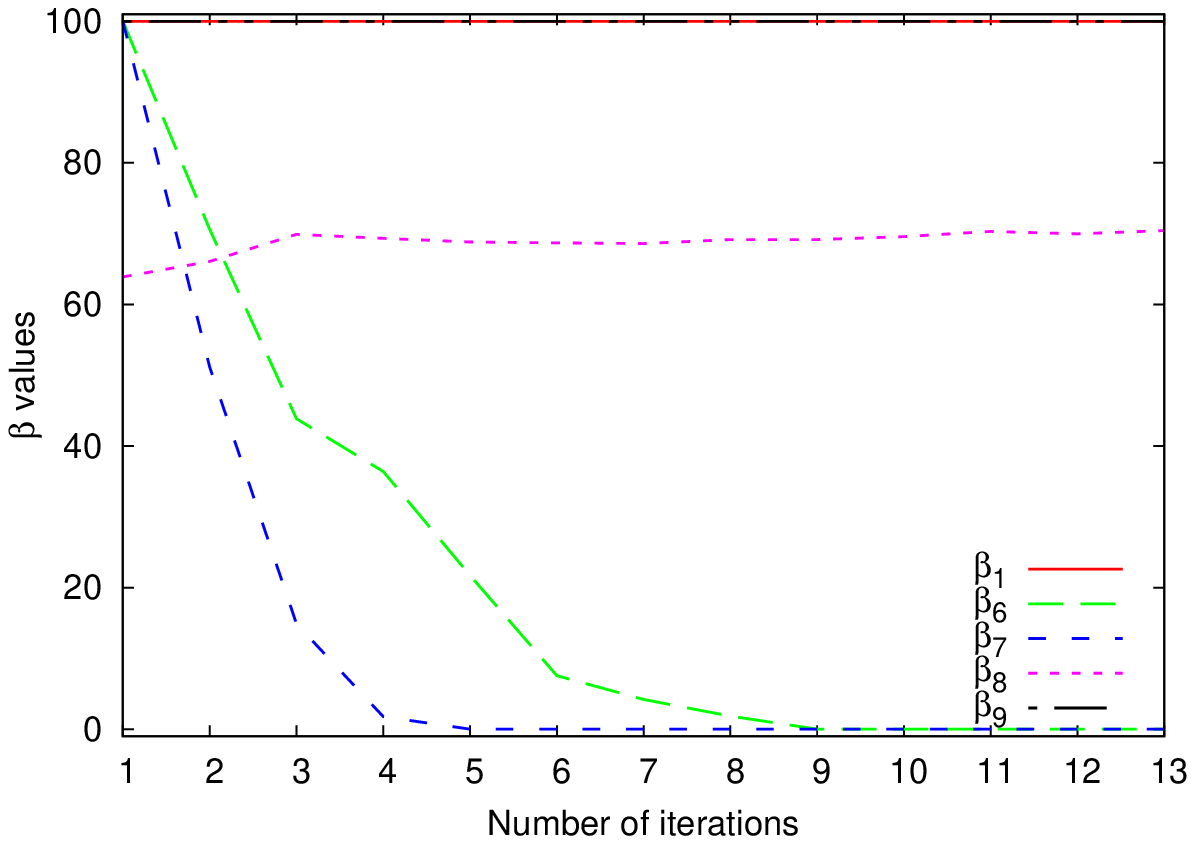}
\includegraphics[width=0.495\textwidth]{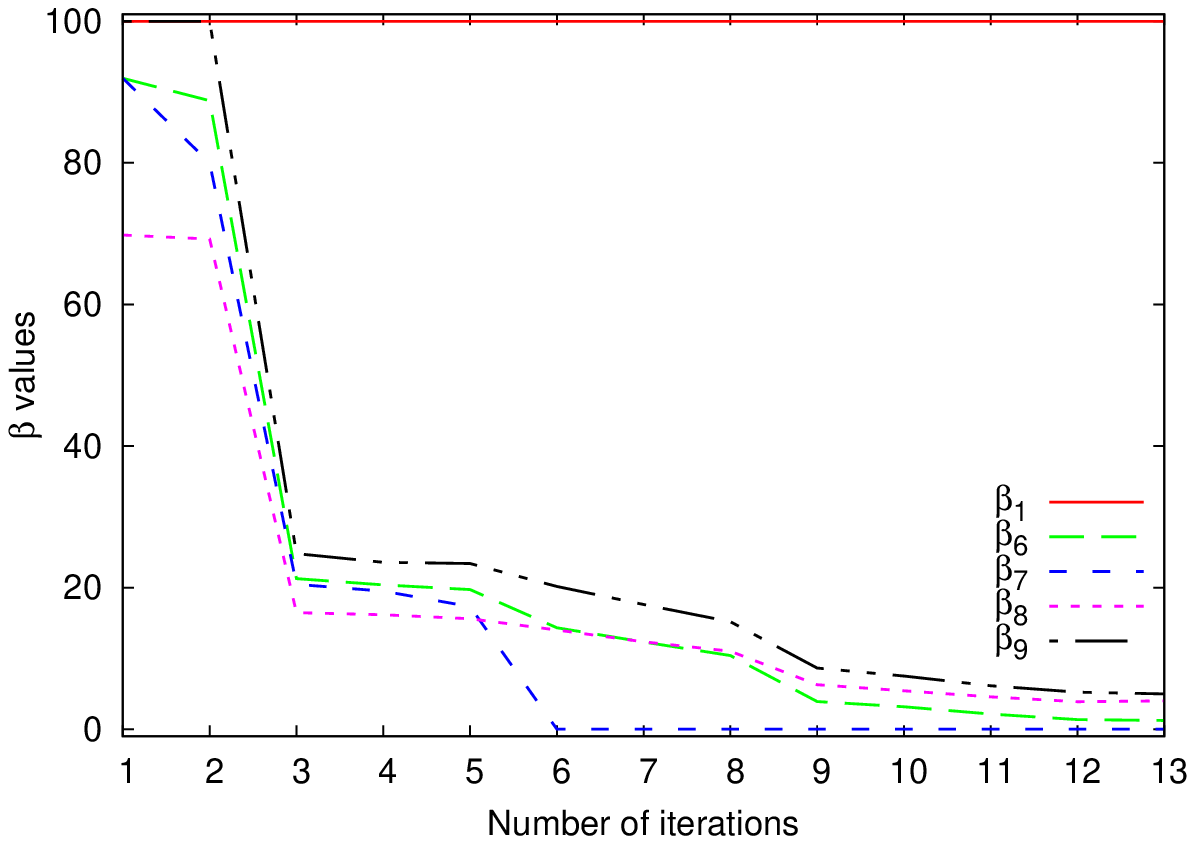}
\end{center}
\caption{Percentage of samples for each task as selected by BTT as a function of FQI iterations when $5000$ samples are available for each source task. \emph{Left:} $100$ target samples available. \emph{Right:} $10000$ target samples available.}\label{F:beta}
\end{figure}

In order to have a better understanding on how BTT trades off between the need for samples and the risk of introducing a large transfer error, in Figure~\ref{F:beta} we show the values of parameters $\beta$ (which represent the percentage of samples transferred from each task) as optimized by the BTT algorithm at each FQI iteration. The tasks considered are the target task $\MDP_1$ and the source tasks $\MDP_6,\MDP_7,\MDP_8,\MDP_9$, each one with $5000$ samples available. Figure~\ref{F:beta} compares the values of $\beta$ in two scenarios: when the available target samples are $100$ (left pane) and $10000$ (right pane). Obviously, BTT always exploits all the target samples ($\beta_1 = 1$). When few target samples are available, BTT transfers high percentages of samples from the source tasks. In particular, it transfers all the samples available from task $\MDP_9$ in each iteration, and also the percentage of samples taken from task $\MDP_8$ is almost constant (about $0.7$). The percentage of samples transferred from tasks $\MDP_6$ and $\MDP_7$ starts from $100\%$ and decreases (with different rates) through iterations reaching zero after iteration 10. This behavior can be explained by the attempt to include as many samples as possible at the earlier iterations when it is still possible to find combinations of sources with a small transfer error. As the iterations continue, no suitable combination of sources is possible and the algorithm is forced to reduce the number of samples from the more different source tasks. On the other hand, when the number of target samples is large enough, we notice that the percentage of samples transferred from all the source tasks drop down after the first FQI iterations. In fact, in this case, BTT exploits a lot of source samples to produce a more accurate approximations only when a very small error is introduced. On the other hand, as the iterations progress, the samples from the source tasks (even when optimally combined) provide a poor approximation of the Q-functions and, as a result, BTT, given the large number of target samples ($10000$), prefers to reduce the number of samples transferred from the source tasks.

\begin{figure}[t]
\begin{center}
\includegraphics[width=0.495\textwidth]{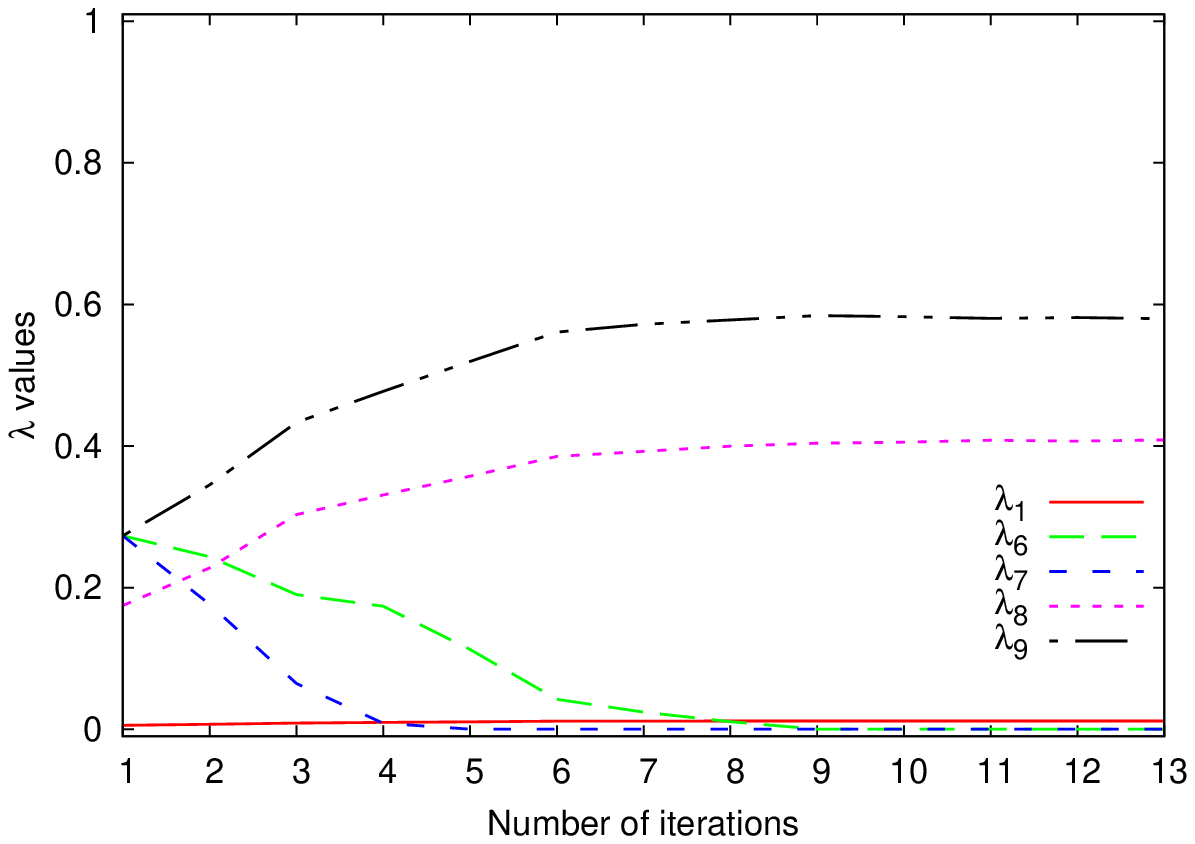}
\includegraphics[width=0.495\textwidth]{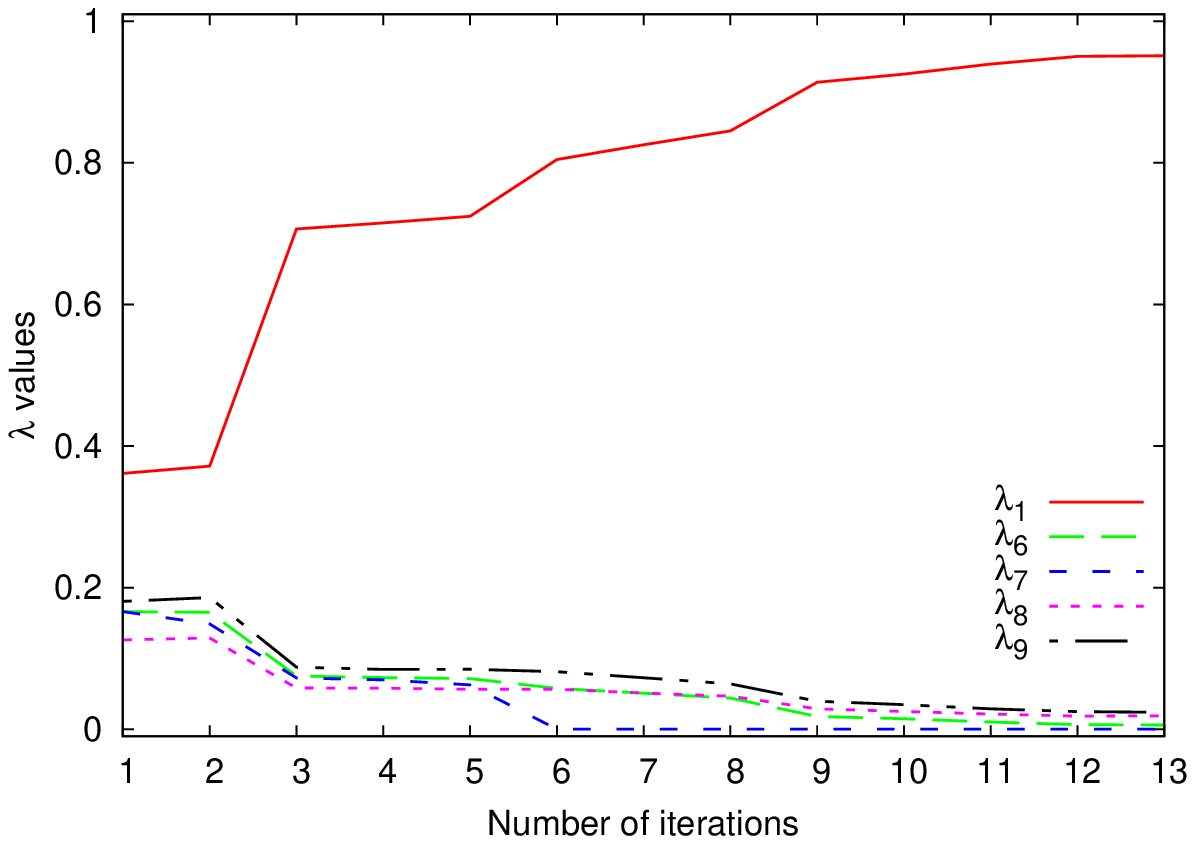}
\end{center}
\caption{Proportions in the combination of task samples induced by BTT $\beta$ parameters as a function of FQI iterations. \emph{Left:} $100$ target samples available. \emph{Right:} $10000$ target samples available.}\label{F:lambda}
\end{figure}

In Figure~\ref{F:lambda} we show the proportions $\lambda$ induced by the weights $\beta$ computed by BTT. When only $100$ target samples are available, BTT tries to compensate the lack of target samples by transferring a large amount of samples from a suitable combination of source tasks, while, when many target samples are available, it considers source samples only when they can guarantee a good approximation of the target Q-functions, otherwise the proportions are changed in favor of the target samples.

Finally, in Figure~\ref{F:samples}, we consider the total number of samples used to train FQI at each iteration under the two scenarios. As expected, at the first iterations, due to the similarity between source tasks and target task, the number of samples provided to FQI by BTT is very large and then it decreases through iterations. It is interesting to notice that the total number of samples selected in the two scenarios are quite similar (in particular starting from the third iteration), which is an effect of the tradeoff realized by the BTT algorithm.

\begin{figure}[t]
\begin{center}
\includegraphics[width=0.495\textwidth]{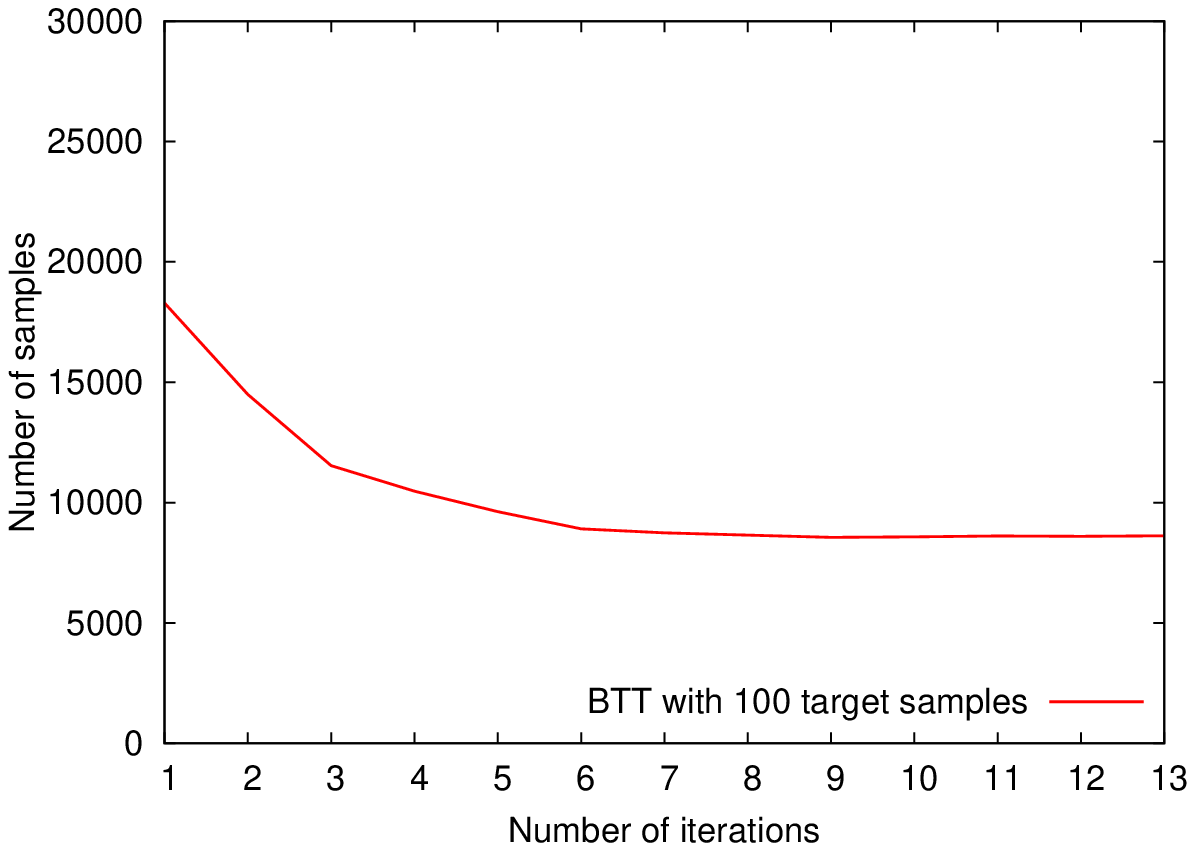}
\includegraphics[width=0.495\textwidth]{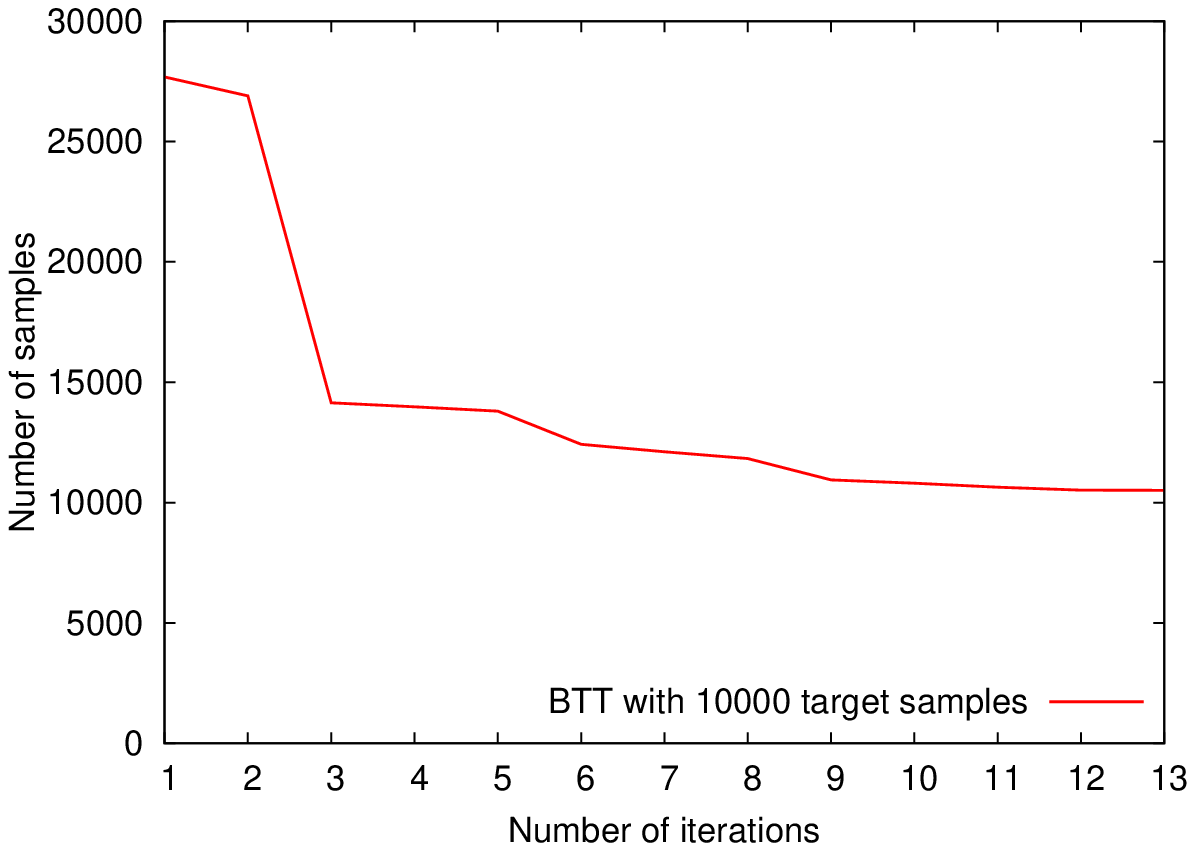}
\end{center}
\caption{Number of samples actually used at each iteration (after the transfer) by FQI. \emph{Left:} $100$ target samples available. \emph{Right:} $10000$ target samples available.}\label{F:samples}
\end{figure}

\subsection{Analysis of parameter $\tau$}

The tradeoff realized by the BTT algorithm is tuned by the parameter $\tau$ multiplying the estimation error. In Figure~\ref{F:tau} we analyze the effect of $\tau$ on the learning performances. Different values of the tradeoff parameter have been tried ($\tau = 0.25, 0.50, 0.75, 1.0$) when both $5000$ samples (left pane) and $10000$ samples (right pane) are available for each source task. As we can notice, BTT is quite robust w.r.t. the choice of the tradeoff parameter. The main differences appear when a small number of target samples is available. In this case, low values of $\tau$ make BTT more concerned about the transfer error and, as a result, it tends to avoid transferring source samples, even if target samples are not enough. On the other hand, with high values of $\tau$, BTT is pushed to use more source samples, and this may negatively affect the performance when several target tasks are available and no combination of source tasks provides a good target approximation.

\begin{figure}[t]
\begin{center}
\includegraphics[width=0.495\textwidth]{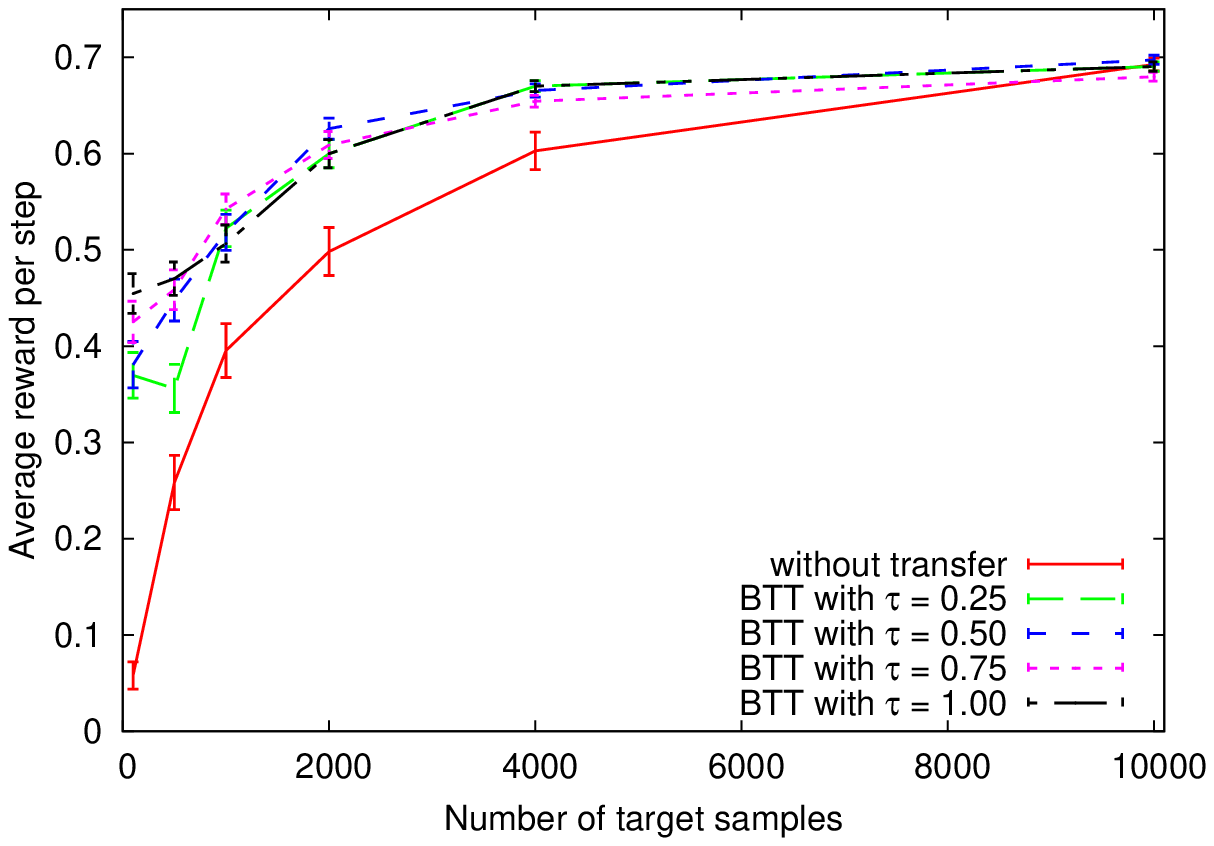}
\includegraphics[width=0.495\textwidth]{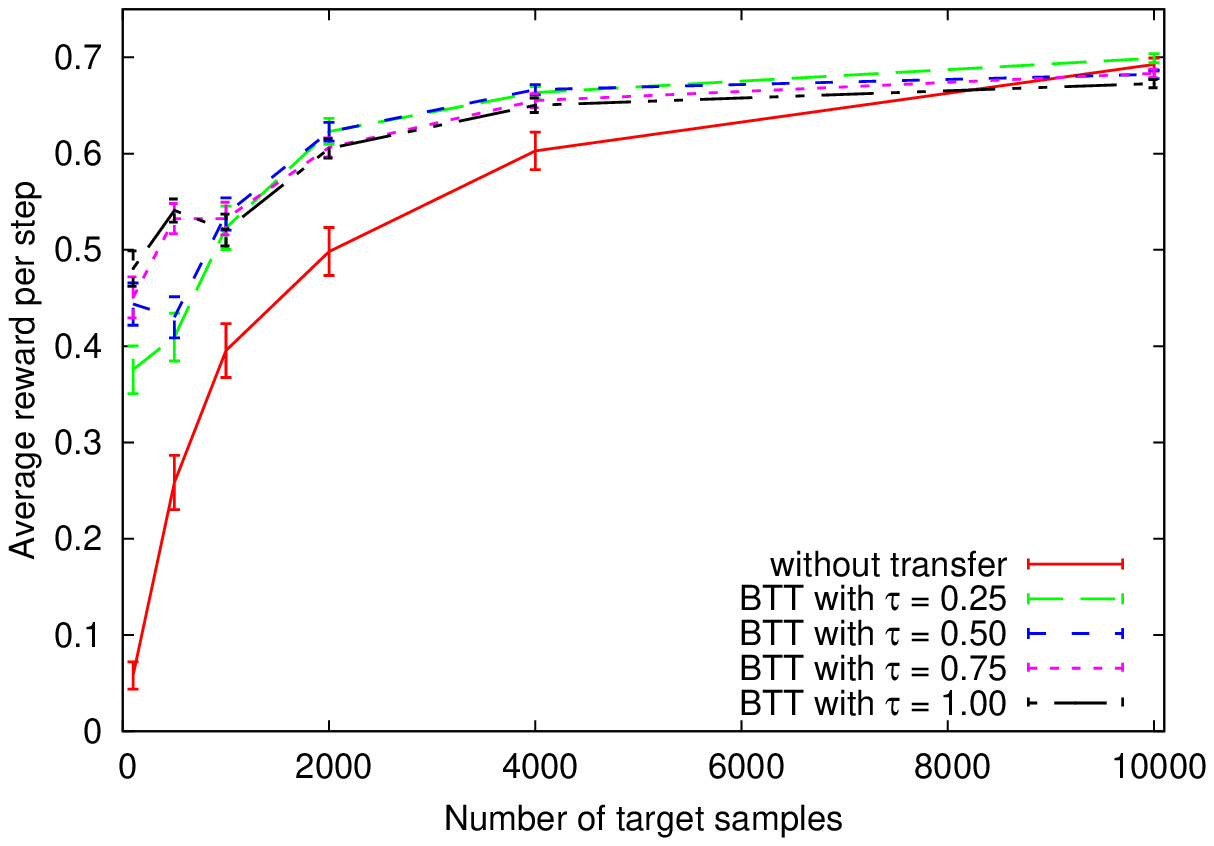}
\end{center}
\caption{Comparison of the performance of FQI using BTT algorithm with different values of the tradeoff parameter $\tau$. \emph{Left:} $5000$ samples available for each source. \emph{Right:} $10000$ samples available for each source.}\label{F:tau}
\end{figure}

\end{document}